\documentclass{article}
\usepackage{graphicx} 
\usepackage[letterpaper,top=2cm,bottom=2cm,left=3cm,right=3cm,marginparwidth=1.75cm]{geometry}
\usepackage[T1]{fontenc}
\usepackage[utf8]{inputenc}
\usepackage[noadjust]{cite}
\usepackage{amsmath, amsfonts, amsthm,amssymb}
\usepackage{graphicx,soul}
\usepackage{bm}
\usepackage{authblk}
\usepackage{float}
\usepackage[colorlinks=true, allcolors=blue]{hyperref}
\newtheorem{theorem}{Theorem}[section]
\newtheorem{lemma}{Lemma}[section]
\newtheorem{corollary}{Corollary}[section]
\newtheorem{assumption}{Assumption}[section]

\newtheorem{proposition}{Proposition}[section]
\usepackage{amsfonts}
\usepackage{color,xcolor}
\theoremstyle{definition}
\newtheorem{definition}[theorem]{Definition}

\usepackage{multirow}
\usepackage{soul}
\theoremstyle{remark}

\usepackage{pdfpages}
\usepackage{subfigure}
\usepackage{ulem}
\usepackage[thicklines]{cancel}
\usepackage{dsfont}
\usepackage{booktabs} 
\usepackage{multirow} 

\usepackage[utf8]{inputenc}
\newcommand{\bM}{\mathbf{M}}
\newcommand{\cB}{\mathcal{B}}

\newcommand{\bI}{\mathbf{I}}

\newcommand{\bTheta}{\mathbf{\Theta}}

\newcommand{\cC}{\mathcal{C}}

\newcommand{\cG}{\mathcal{G}}

\newcommand{\bZ}{\mathbf{Z}}
\newcommand{\bW}{\mathbf{W}}
\newcommand{\bs}{\mathbf{\Sigma}}

\usepackage{marvosym}
\newcommand{\mail}{\textsuperscript{\Letter}}

\title{Second-Order Path Kernel Interpolation Formulas\\ in Machine Learning}
\author[1]{Jin Guo}
\author[1]{Roy Y. He}
\author[2,\mail]{Jean-Michel Morel}
\affil[1]{Department of Mathematics, City University of Hong Kong, Hong Kong}
\affil[2,\mail]{Division of Industrial Data Science, Lingnan University, Hong Kong }

\date{}

\begin{document}

\maketitle
\begin{abstract}
Understanding how training data shape neural network predictions is a central problem in modern learning theory. In 2020, Pedro Domingos proposed an interpolation formula valid for every model learned by deterministic gradient descent. It expresses the model's prediction as an integral, along the optimization path, of a data-dependent kernel that aligns the model's gradients at the test and training data. Such a first-order characterization remains valid for models trained with batch-based stochastic optimization. In this paper, we develop second-order forms of these interpolation formulas. We show that the leading path-kernel interpolation is supplemented by a curvature-weighted interpolation term. For stochastic gradient descent, an additional sampling-induced component appears, coupling the curvature of the prediction with the covariance of mini-batch gradient noise. We also extend the representation to stochastic gradient descent with momentum, where the interpolation structure is preserved but with  the  weights modified by a memory-related factor. Moreover, we establish a concentration estimate for the terminal prediction, identifying the fluctuation scale around the expected second-order representation. Together, these results provide a refinement of the path-kernel interpretation of neural network prediction.
\end{abstract}

\section{Introduction}
Deep learning has achieved remarkable empirical success across a wide range of tasks~\cite{lecun2015deep,schmidhuber2015deep}. This success has motivated growing interest in characterizing models' predictions under gradient-descent training and explaining their behavior on unseen data. Kernel perspectives provide an important approach to this problem. The neural tangent kernel (NTK)~\cite{jacot2018neural,novak2019neural} interprets the learning process as a linearized model around initialization in the infinite-width regime~\cite{lee2019wide}. However, fixed-kernel or purely linearized descriptions are incomplete for finite-width networks obtained through optimization, where the effective kernel may evolve along the training trajectory and depend on the data.

An alternative kernel perspective was developed by Domingos~\cite{domingos2020every}, who showed that models trained by gradient descent, including neural networks, induce path kernels that determine the model's prediction. Since path kernels accumulate gradient information from all training points throughout the training process and match it against the gradient at the query in a Hilbert space, this theory yields a clear picture of the influence of training data. More specifically, let \(f(\cdot,\bTheta): \mathbb{R}^p\to \mathbb{R}^m\) be a differentiable model trained on \(\{(x_n,y_n^*)\}_{n=1}^N\subset\mathbb{R}^p\times\mathbb{R}^m\) with differentiable loss
    \(L(\bTheta) = \frac{1}{N}\sum_{n=1}^N \ell(f(x_n,\bTheta), y_n^*)\) and learning rate $\eta>0$ by gradient descent; then for any input \(x \in \mathbb{R}^p\), the model output at time
    \(T\) satisfies
    \begin{equation}\label{eq:determ-domingos}
        \lim_{\eta\to 0}f(x,\bTheta_T)
        = - \frac{1}{N}\sum_{n=1}^N 
\int_0^T \, K_t(x,x_n) \frac{\partial \ell}{\partial f}\bigl(f(x_n,\bTheta_t),y_n^*\bigr)\, dt + f(x,\bTheta_0),
    \end{equation}
where $$K_t(x,x') := \bigl\langle \nabla_{\bTheta} f(x,\bTheta_t),\, 
\nabla_{\bTheta} f(x',\bTheta_t) \bigr\rangle,\qquad x,x'\in\mathbb{R}^p,$$ is the \textit{gradient kernel}.
In~\cite{guo2026interpolation}, Guo et al. extended~\eqref{eq:determ-domingos} to a non-asymptotic, stochastic setting to characterize the predictions of models learned through mini-batch training. In particular, they show that, under a first-order weak approximation~\cite{li2019stochastic} of stochastic batch gradient descent, the expected network output admits a representation analogous to~\eqref{eq:determ-domingos} with optimizer-specific weights.

The above extension is closely related to the approximation of stochastic iterative algorithms. Classical theory studies their continuous-time limits~\cite{kushner1987stochastic,ljung2012stochastic}, and diffusion approximations have been used to analyze SGD through weak convergence and the evolution of its probability distribution~\cite{hu2017diffusion}. More recently, stochastic modified equations have provided a systematic weak-approximation framework for capturing finite-step-size and mini-batch effects beyond the leading-order continuous-time limit~\cite{li2019stochastic}. This viewpoint has also been extended to adaptive methods such as RMSprop and Adam~\cite{malladi2022sdes}, and to distributed and local variants of SGD~\cite{gu2023and}. These works show that \textit{higher-order} terms are not merely discretization remainders but can encode meaningful properties of the training dynamics.

Meanwhile, Domingos' formula~\eqref{eq:determ-domingos}, as well as the representations established in~\cite{guo2026interpolation}, captures only the leading tangent geometry of the training path, leaving the influence of higher-order geometry unaddressed. For instance, the curvature of the loss function in parameter space has long been recognized as a central element in the optimization of neural networks. Empirical Hessian analyses show that neural loss landscapes have highly structured spectra, often with a large near-zero bulk and a small number of outlier directions associated with data-dependent structure~\cite{sagun2017empirical}. Loss-landscape visualization further reveals that the local geometry around trained solutions carries important information about the stability of neural network minima and their generalization behavior~\cite{li2018visualizing}. Related work based on local entropy similarly highlights the role of wide valleys and flat local geometry in the solutions found through neural network training~\cite{chaudhari2019entropy}, while large-scale Hessian spectrum studies show that curvature evolves throughout training and shapes optimization behavior~\cite{ghorbani2019investigation}. These findings also provide a foundation for deep learning privacy~\cite{ravikumar2024curvature}. Hence, investigating the higher-order influences on a model's predictions can yield valuable insights, yet remains under-explored.

Motivated by this gap, we develop the second-order counterpart of the first-order path-kernel representation. First-order path kernels capture the leading tangent-feature alignment, but they do not describe how second-order geometry modifies the contribution of training samples to test predictions. We show that curvature plays a direct, prediction-level role: it reweights the pathwise contribution of each training sample to test outputs. Moreover, we find that this curvature effect is a common mechanism across gradient-based methods, including gradient descent (GD), stochastic gradient descent (SGD)~\cite{d8d62392-9a37-31e7-ad3b-37a6f6ee8ef6}, and SGD with momentum (SGDM)~\cite{sutskever2013importance}, all of which contain a curvature-weighted correction to the first-order path-kernel term. Our theory thus lifts this higher-order viewpoint to the prediction level, yielding explicit second-order corrections to the pathwise formula for test outputs.

More interestingly, once we move beyond first order in the stochastic setting, we find an additional mechanism that is absent in deterministic training. In SGD and SGDM, mini-batch sampling produces terms that couple sampling fluctuations with the curvature of the network in parameter space. This connects our result to a line of work showing that SGD noise, gradient covariance, and curvature jointly influence the minima reached by stochastic training~\cite{jastrzkebski2017three,zhu2018anisotropic,li2026hessian}. Unlike these parameter-space analyses, which typically study stability, escape behavior, or Hessian-based sharpness~\cite{zhu2018anisotropic,xie2020diffusion,NEURIPS2022_1e55c38d,dinh2017sharp}, our result offers a prediction-level perspective on the influence of the flatness or sharpness of local minima. Curvature matters not only because it affects the motion of the parameters near a minimum, but also because it directly reshapes the pathwise contribution of training samples to predictions.

The main contributions of this paper are summarized as follows.
\begin{itemize}
\item We identify a curvature-weighted interpolation term shared by GD, SGD, and SGDM, which reweights the training residuals through the Hessian of the loss. This extends the first-order path-kernel formula from a leading tangent-feature description to a refined output expansion.
\item For stochastic methods, we identify sampling-induced corrections that are absent in deterministic GD. These terms couple mini-batch fluctuations with the curvature of the network in parameter space, thereby making the dependence of the model's predictions on batch size and curvature explicit. 
\item For SGDM, we further show that momentum preserves the interpolation structure while reshaping the temporal weights through a memory kernel.
\item We support the sampling-induced representation with a localized concentration estimate for the terminal prediction, together with numerical experiments that verify the predicted remainder scaling and sensitivity effects.
\end{itemize}
This paper is organized as follows. Section~\ref{preliminaries} introduces the basic notation and preliminaries, including path kernels and the SDE approximation of SGD. Section~\ref{main_result} presents the main second-order output representations for GD, SGD, and SGDM. Section~\ref{inplications} discusses several implications of the theory, including curvature-weighted interpolation, the concentration of model predictions, and the role of sampling-induced variation and batch size. Section~\ref{experiments} provides numerical experiments that verify  predicted scaling orders and illustrate the effect of batch size and curvature-based batch selection on prediction sensitivity. Section~\ref{conclusion} concludes the paper. The appendix contains the extension to learning-rate schedules and the auxiliary lemmas used in the proofs.

\section{Preliminaries}\label{preliminaries}
In this section, we define the key mathematical objects and provide the necessary background.
\subsection{Parameterized learning models}
Consider a learning model \(f: \mathbb{R}^p \times \mathbb{R}^d \to \mathbb{R}^m\) parameterized by \(\bTheta \in \mathbb{R}^d\). This includes classical methods such as support vector machines (SVMs) as well as modern neural networks of any architecture. Let the training data \(\{(x_n,y_n^*)\}_{n=1}^N \subseteq \mathbb{R}^p \times \mathbb{R}^m\) be drawn independently from an unknown distribution, and denote by \(\mathcal{P}(x)\) the marginal distribution of \(x\). For an input \(x\), the model output is \(y(x) = f(x, \bTheta) \in \mathbb{R}^m\), and the loss function is \(\ell(\cdot,\cdot): \mathbb{R}^m \times \mathbb{R}^m \to \mathbb{R}\). Training the model corresponds to finding an optimal parameter \(\bTheta\) that minimizes the expected risk
\[
    \int \ell\bigl(f(x,\bTheta),\, y^*(x)\bigr) \, d\mathcal{P}(x),
\]
where \(y^*(x)\) denotes the true response at \(x\). Since the data distribution is unknown, the expected risk cannot be minimized directly. Instead, one minimizes the empirical risk
\[
    L(\bTheta) = \frac{1}{N} \sum_{n=1}^N \ell\bigl(f(x_n,\bTheta),\, y_n^*\bigr).
\]
Standard gradient descent on the parameters \(\bTheta\) yields the iterative update
\[
    \bTheta_k = \bTheta_{k-1} - \eta \, \nabla L(\bTheta_{k-1}),
\]
where \(k\) is the iteration index, \(\eta > 0\) is the learning rate, and \(\bTheta_0\) is the vector of initial parameters.
For large datasets, computing the full gradient at each iteration can be expensive. A standard alternative is stochastic gradient descent (SGD), which approximates the full-batch loss using randomly sampled mini-batches. Let \(|\mathcal{B}| \le N\) denote the batch size, \(\Gamma\) the set of all subsets of \(\{1,\dots,N\}\) of size \(|\mathcal{B}|\), and \(\mathcal{B}\) a uniformly random element of \(\Gamma\). The mini-batch loss is
\[
L_{\mathcal{B}}(\bTheta) = \frac{1}{B}\sum_{i\in \mathcal{B}} \ell\bigl(f(x_i,\bTheta),\, y_i^*\bigr).
\]

Since \(\Gamma\) is finite and \(\mathcal{B}\) is uniformly distributed, the expectation \(\mathbb{E}_{\mathcal{B}}[L_{\mathcal{B}}(\bTheta)]\) is well defined and coincides with the empirical risk, \(\mathbb{E}_{\mathcal{B}}[L_{\mathcal{B}}(\bTheta)] = L(\bTheta)\). Interchanging the gradient with the finite sum then gives
\[
\mathbb{E}_{\mathcal{B}}\bigl[\nabla_{\bTheta}L_{\mathcal{B}}(\bTheta)\bigr] 
= \nabla_{\bTheta}\,\mathbb{E}_{\mathcal{B}}\bigl[L_{\mathcal{B}}(\bTheta)\bigr] 
= \nabla_{\bTheta}L(\bTheta),
\]
so the mini-batch gradient is an unbiased estimator of the full-batch gradient. Let \(\{\mathcal{B}_k : k = 0,1,2,\dots\}\) be a sequence of independent and identically distributed (i.i.d.) uniform random batches. Starting from an initial point \(\bTheta_0\in\mathbb{R}^d\), the SGD update at step \(k\) is
\begin{align}
    \bTheta_{k+1} = \bTheta_k - \eta \, \nabla L_{\mathcal{B}_k}(\bTheta_k),\label{sgd}
\end{align}
where \(\eta > 0\) is the learning rate.

\subsection{Gradient kernel and path kernel}
Domingos~\cite{domingos2020every} established a first-order kernel representation~\eqref{eq:determ-domingos} for gradient descent. The key components of this interpolation theory are the gradient kernel and the optimization-induced path kernel.
\begin{definition}[Gradient kernel and path kernel]
Let $\bTheta_t$, $t \in [0,T]$, denote the parameter trajectory of the 
gradient flow~\eqref{eq: gradient-flow}. The \textit{gradient kernel} 
between two inputs $(x, x')\in \mathbb{R}^{p}\times\mathbb{R}^p$ at 
time $t$ is
\begin{equation}\label{eq_gradient_kernel}
K_t(x,x') := \bigl\langle \nabla_{\bTheta} f(x,\bTheta_t),\, 
\nabla_{\bTheta} f(x',\bTheta_t) \bigr\rangle.
\end{equation}
The \textit{path kernel} is its time integral over the training trajectory:
\begin{equation}\label{eq_path_kernel}
K_{\mathrm{path}}(x,x') := \int_0^T K_t(x,x') \, dt.
\end{equation}
\end{definition}
Intuitively, the gradient kernel~\eqref{eq_gradient_kernel} measures the instantaneous similarity between the model's gradients at two inputs, and the path kernel~\eqref{eq_path_kernel} accumulates this similarity over the entire training trajectory.

\subsection{SDE approximation of SGD}

Building on the first-order stochastic extension of Domingos' theorem, we now focus on identifying and characterizing the second-order structure of the output formula. Rather than re-establishing the stochastic representation itself, we investigate how finite-learning-rate effects, loss curvature, and mini-batch sampling noise contribute additional second-order corrections. To this end, we use the stochastic differential equation (SDE) approximation of SGD as a continuous-time framework for deriving and interpreting these higher-order terms.
\begin{assumption}\label{assumption}
For every integer \(n=1,\dots,N\), the per-sample loss \(\ell(f(x_n,\bTheta),y_n^*)\) is continuously differentiable in \(\bTheta\), and its gradient satisfies a uniform local bound: for each \(R > 0\) there exists a constant \(M_R > 0\) such that
\[
\max_{\|\bTheta\| \le R} \|\nabla_{\bTheta}\, \ell(f(x_n,\bTheta),y_n^*)\| \le M_R \quad \text{for all } n=1,\dots,N.
\]
In addition, for each fixed test point \(x\in\mathbb{R}^p\), the parameter gradient
\(\nabla_\bTheta f(x,\bTheta)\) is locally bounded in \(\bTheta\).
\end{assumption}
This assumption holds for any smooth network and loss on a bounded domain. It does, however, exclude networks with nonsmooth activations such as ReLU and nonsmooth losses such as the \(\ell_1\) loss.
With the SGD iterates~\eqref{sgd}, the random mini-batch sampling is inherently discrete. In contrast, Domingos' theorem (Theorem~\ref{thm_domingo}) relies on a continuous gradient flow~\eqref{eq: gradient-flow} and does not directly extend to the stochastic, finite-learning-rate setting. In~\cite{guo2026interpolation}, a first-order continuous-time SDE approximation of stochastic gradient algorithms~\cite{li2019stochastic,hu2017diffusion,mandt2016variational} was used to extend the representation to data-driven models trained with stochastic gradient descent (SGD). We briefly review the relevant concepts in their general form.
Let $T > 0$, $\eta \in (0, 1 \wedge T)$ (where $a \wedge b := \min\{a, b\}$), and let $\alpha \geq 1$ be an integer. Set $K = \lfloor T / \eta \rfloor$ (where $\lfloor \cdot \rfloor$ denotes the floor function). Let $\cG$ be the set of continuous functions $g : \mathbb{R}^d \to \mathbb{R}$ satisfying
\[
    |g(z)| \leq \kappa_1 \bigl( 1 + |z|^{2\kappa_2} \bigr),
    \quad z \in \mathbb{R}^d,
\]
for some $\kappa_1, \kappa_2 > 0$. For $\alpha \geq 1$, define
\[
    \cG^\alpha := \bigl\{ g \in \mathcal{C}^\alpha(\mathbb{R}^d) : 
    \partial^\beta g \in \cG \text{ for all } |\beta| \leq \alpha \bigr\},
\]
where $\beta$ is a multi-index. Thus $\cG^\alpha$ is a subset of $\cC^{\alpha}$, the space of $\alpha$-times continuously differentiable functions.
\begin{definition}[$\alpha$-th order weak approximation~\cite{li2019stochastic}]\label{def:approx}
     We say that a continuous-time stochastic process $\{\bZ_t : t \in [0, T]\}$ in $\mathbb{R}^d$ is an $\alpha$-th order weak approximation of a discrete stochastic process $\{\bTheta_k : k = 0, \dots, K\}$ in $\mathbb{R}^d$ if for every $g \in \cG^{\alpha+1}$, there exists a positive constant $C$, independent of $\eta$, such that
    \begin{equation*}
        \max_{k = 0, \dots, K} \bigl| \mathbb{E}[g(\bTheta_k)] - \mathbb{E}[g(\bZ_{k\eta})] \bigr| \leq C\eta^\alpha.
    \end{equation*}
\end{definition}
In the context of Definition~\ref{def:approx}, functions from \(\cG^{\alpha+1}\) serve as test functions. This space collects sufficiently smooth functions with at most polynomial growth, which is a condition mild enough to include the observables of interest, yet strong enough to ensure integrability and to justify the It\^{o}--Taylor expansion rigorously. This framework is widely employed to analyze the discretization errors of stochastic differential equations (SDEs) when modeling SGD via continuous-time approximations~\cite{kloeden2013numerical,feng2019uniform,li2019stochastic}. Crucially, the resulting error bounds hold uniformly over all SGD iterations up to \(\lfloor T/\eta\rfloor\) steps, not merely at the final iteration.

\section{Main Results}\label{main_result}
In this section, we present our main results, which characterize the role of second-order geometry in the predictions of parameterized learning models. We summarize our findings as follows.
\begin{itemize}
\item We show that, in addition to the path-kernel-induced gradient alignment in~\eqref{eq:determ-domingos}, an analogous alignment takes place in a different Hilbert space, one modified by the local curvature of the loss landscape through the Hessian (Theorem~\ref{thm_domingo}). This constitutes a second-order extension of~\cite{domingos2020every}, valid for  deterministic gradient descent.
\item Passing to the more realistic stochastic descent framework, we find that, under mild conditions, this second-order extension carries over to the mean predictions of models trained via SGD, but with a new sampling-induced perturbation (Theorem~\ref{2-order Domingo}).
\item We show that the momentum in SGDM induces additional tangent alignments in Hilbert spaces generated by time-transported kernels~\eqref{eq:def_memory_kernel_final}--\eqref{eq:def_cur_kernel_final}, which vanish when the network gradient and the loss Hessian remain bounded (Theorem~\ref{2 order SGDM}).
\end{itemize}
We focus primarily on the role and implications of the second-order terms for the model's predictions, and refer the reader to~\cite{domingos2020every,guo2026interpolation} for discussion of the first-order terms.

\subsection{Second-order Domingos Theorem}
To capture the effects of finite step size more accurately, we replace the plain gradient flow with the second-order modified equation~\cite{barrett2020implicit,di2023backward} of gradient descent, as is standard in backward error analysis:
\begin{equation}\label{eq: gradient-flow}
\dot{\bTheta}_t
=
-\nabla L(\bTheta_t)
-\frac{\eta}{2}\nabla^2 L(\bTheta_t)\nabla L(\bTheta_t)
+O(\eta^2).
\end{equation}
\begin{definition}[Curvature-induced kernel]
The curvature-induced kernel between two data points \(x,x'\) is the Hessian-weighted inner product of the model gradients at the two points along the parameter trajectory induced by gradient descent; that is,
\begin{align}                        
K_t^{\mathrm{cur}}(x,x')
&:=
\left\langle
\nabla_{\bTheta} f(x,\bTheta_t),
\nabla_{\bTheta} f(x',\bTheta_t)
\right\rangle_{\nabla_{\bTheta}^2 L(\bTheta_t)}\\
&\;=
\nabla_{\bTheta} f(x,\bTheta_t)^{\top}
\nabla_{\bTheta}^2 L(\bTheta_t)
\nabla_{\bTheta} f(x',\bTheta_t).\notag
\end{align}
\end{definition}
The kernel \(K_t^{\mathrm{cur}}(x,x')\) measures how the alignment between the two tangent features is reweighted by the local curvature of the loss landscape during training. In contrast to the standard gradient inner product, this quantity assigns greater weight to directions associated with larger curvature of \(L\) at \(\bTheta_t\).
The following theorem shows that the output of any model trained by 
gradient descent can be expressed as a weighted sum of path-kernel and curvature-kernel contributions evaluated between the test input and each training point.
\begin{theorem}\label{thm_domingo}
    Consider a differentiable model \(f(\cdot,\bTheta)\) trained on a set \(\{(x_n,y_n^*)\}_{n=1}^N\) via gradient descent on the empirical loss
    \(L(\bTheta) = \frac{1}{N}\sum_{n=1}^N \ell(f(x_n,\bTheta), y_n^*)\) with learning rate \(\eta > 0\). The model output satisfies
    \begin{align}\label{continuous Domingos}
   f(x,\bTheta_K)
        =&\,f(x,\bTheta_0)-\frac{1}{N}\sum_{n=1}^N \int_0^{T} \underbrace{K_t(x,x_n)\frac{\partial \ell}{\partial f}(f(x_n,\bTheta_t),y_n^*)}_{\text{path-kernel term}}\,\mathrm{d}t\notag\\&-\frac{\eta}{2N}\sum_{n=1}^N\int_0^{T} \underbrace{K_t^{\mathrm{cur}}(x,x_n)\frac{\partial \ell}{\partial f}(f(x_n,\bTheta_t),y_n^*)}_{\text{curvature-kernel term}}\,\mathrm{d}t+O(\eta^2).
    \end{align}
\end{theorem}
\begin{proof}
    Let $\bTheta_t\in \mathbb{R}^d$ denote the parameter trajectory of the model $f(\cdot,\bTheta)$. Following the standard modified equation~\eqref{eq: gradient-flow} for gradient descent, we approximate the discrete update $\bTheta_{k+1}=\bTheta_k-\eta\nabla L(\bTheta_k)$ by continuous-time dynamics
    \begin{align}
        \label{eq:second_order_dynamics}\dot{\bTheta}_t=-\nabla_\bTheta L(\bTheta_t)
-\frac{\eta}{2}\nabla^2_\bTheta L(\bTheta_t)\nabla_\bTheta L(\bTheta_t)
+O(\eta^2).
    \end{align}
    By the chain rule and~\eqref{eq:second_order_dynamics}, we obtain
    \begin{align*}
        \frac{\mathrm{d}f}{\mathrm{d}t}(x,\bTheta_t)
        &=\nabla_\bTheta f(x,\bTheta_t)^\top\dot{\bTheta}_t
        \\
        &=\nabla_\bTheta f(x,\bTheta_t)^\top\left( -\nabla_\bTheta L(\bTheta_t)
-\frac{\eta}{2}\nabla^2_\bTheta L(\bTheta_t)\nabla_\bTheta L(\bTheta_t)
+O(\eta^2)\right)
        \\
        &=-\frac{1}{N}\sum_{n=1}^N\langle\nabla_\bTheta f(x,\bTheta_t),\nabla_\bTheta f(x_n,\bTheta_t)\rangle \frac{\partial \ell}{\partial f}(f(x_n,\bTheta_t),y_n^*)\\
        &\quad
        -\frac{\eta}{2N}\sum_{n=1}^N \langle\nabla_\bTheta f(x,\bTheta_t),\nabla_\bTheta f(x_n,\bTheta_t)\rangle_{\nabla^2_\bTheta L(\bTheta_t)}\frac{\partial \ell}{\partial f}(f(x_n,\bTheta_t),y_n^*)+O(\eta^2),
    \end{align*}
    where $\langle\nabla_\bTheta f(x,\bTheta_t),\nabla_\bTheta f(x_n,\bTheta_t)\rangle_{\nabla^2_\bTheta L(\bTheta_t)}=\nabla_\bTheta f(x,\bTheta_t)^\top \nabla^2_\bTheta L(\bTheta_t) \nabla_\bTheta f(x_n,\bTheta_t)$. Integrating both sides from $0$ to $T$ yields the result.
\end{proof}
The curvature term also admits an interpolation interpretation. Here, $K_t^{\mathrm{cur}}(x,x_n)$ acts as a curvature-weighted coefficient applied to $\frac{\partial \ell}{\partial f}\bigl(f(x_n,\bTheta_t),y_n^*\bigr)$, which is interpretable as a  training residual . Since this term carries a factor of $\eta$, it represents a second-order, curvature-dependent interpolation correction rather than the contribution of the leading path-kernel. A natural question is whether Domingos' theorem continues to hold under a learning-rate schedule. The answer is yes, and we prove in Appendix~\ref{sec:second-scheduler} that the schedule reweights both the path-kernel term and its curvature-dependent correction along the optimization trajectory.

\subsection{Sampling-induced kernel from random batches}

To derive a refined representation beyond the first-order path-kernel expansion, we employ a second-order weak approximation of SGD. For a given learning rate~$\eta>0$, the discrete-time iterates are weakly approximated by a continuous process that retains the first-order drift-diffusion structure while incorporating second-order corrections generated by loss curvature and mini-batch fluctuations. The following lemma provides this approximation and yields an $O(\eta^2)$ bound on the weak error between the discrete iterates and the continuous process.
\begin{lemma}[Second-order weak approximation of SGD~\cite{li2019stochastic}]\label{2 order}
    Let $T>0$, $\eta\in (0,1\wedge T)$, and set $K=\lfloor T/\eta\rfloor$. Let $\{\bTheta_k : k\geq 0\}$ be the SGD iterates defined in~\eqref{sgd}. Assume that $L$ is twice continuously differentiable, $\nabla|\nabla L|^2$ and $\nabla L_\mathcal{B}$ are Lipschitz, $L\in \mathcal G^4$, and the test function $g\in \mathcal G^3$.
    Define $\{\bZ_t:t\in [0,T]\}$ as the stochastic process satisfying the SDE
    \begin{equation}\label{eq:2 order}
        d\bZ_t=-\nabla\!\left(L(\bZ_t)+\tfrac{1}{4}\eta\,|\nabla L(\bZ_t)|^2\right)dt+\sqrt{\eta}\,\bs_{|\mathcal B|}^{1/2}(\bZ_t)\,dW_t,\qquad \bZ_0=\bTheta_0,
    \end{equation}
    where $\bs_{|\mathcal{B}|}(\bZ)=\mathbb{E}_\mathcal{B}\!\left[(\nabla L_\mathcal{B}(\bZ)-\nabla L(\bZ))(\nabla L_\mathcal{B}(\bZ)-\nabla L(\bZ))^\top\right]$. Then $\{\bZ_t:t\in [0,T]\}$ is a second-order weak approximation of SGD; that is, there exists a constant $C>0$, independent of $\eta$, such that
    \begin{equation*}
        \max_{k=0,\dots,K}\bigl|\mathbb{E}[g(\bTheta_k)]-\mathbb{E}[g(\bZ_{k\eta})]\bigr|\leq C\eta^2.
    \end{equation*}
\end{lemma}

Lemma~\ref{2 order} ensures that, under the regularity assumptions stated above, the continuous process~\eqref{eq:2 order} captures the expected behavior of SGD uniformly over all $K$ iterates, to second order in the learning rate. This refined continuous-time representation serves as the foundation for deriving our main results.
For a mini-batch gradient estimator $L_{\mathcal{B}}(\bZ)$ with batch size $|\mathcal{B}|$, the gradient-noise covariance scales as $O(|\mathcal{B}|^{-1})$ at a fixed parameter value $\bTheta$~\cite{mandt2017stochastic}. Under i.i.d. sampling with replacement, one has the exact identity
\[
\bs_{|\mathcal{B}|}(\bZ_s)=\frac{1}{|\mathcal{B}|}\,\bs_{\mathrm{single}}(\bZ_s),
\]
where $\bs_{\mathrm{single}}(\bZ_s):=\frac1N\sum_{n=1}^N
\big(\nabla \ell(x_n,\bZ_s)-\nabla L(\bZ_s)\big)
\big(\nabla \ell(x_n,\bZ_s)-\nabla L(\bZ_s)\big)^\top$ is the single-sample gradient covariance~\cite{wu2020noisy}. Under uniform sampling without replacement from a dataset of size $N$, the exact formula becomes
\begin{equation*}
\bs_{|\mathcal{B}|}(\bZ_s)=\frac{N-|\mathcal{B}|}{|\mathcal{B}|(N-1)}\,\bs_{\mathrm{single}}(\bZ_s),
\end{equation*}
which reduces to the same $|\mathcal{B}|^{-1}$ scaling, up to the finite-population correction factor. In particular, when $|\mathcal{B}|\ll N$, the covariance is well approximated by a term proportional to $1/|\mathcal{B}|$.
The following theorem identifies how this batch-size-dependent covariance contributes to the expected prediction through the second-order sampling-induced term.

\begin{theorem}[Stochastic Domingos' Theorem (SGD)]\label{2-order Domingo}
Consider a learning model $y = f(x, \bTheta)$, with $f(x, \cdot) \in \mathcal G^3$. The parameter $\bTheta$ of the model is learned from the dataset $\{(x_n, y_n^*)\}_{n=1}^N$ by SGD with learning rate $\eta$ and Assumption~\ref{assumption} holds. In addition, assume that  $\mathbb{E}[L_{\mathcal{B}}(\bTheta)] \in \mathcal G^4$. Then we have
\begin{align}   
    \mathbb{E}[f(x,\bTheta_K)] =&\,
    f(x,\bTheta_0) 
    - \frac{1}{N}\sum_{n=1}^N\mathbb{E}\!\Bigg[ \int_0^{T} \underbrace{ K_s(x,x_n) \frac{\partial \ell}{\partial f}\bigl(f(x_n,\bZ_s), y_n^*\bigr)}_{\text{path-kernel term}} \, ds \Bigg] \notag \\
    &-\frac{\eta}{2N}\,\sum_{n=1}^N
    \mathbb{E}\!\Bigg[ \int_0^{T} \underbrace{ K_s^{\operatorname{cur}}(x,x_n) \frac{\partial \ell}{\partial f}\bigl(f(x_n,\bZ_s), y_n^*\bigr) }_{\text{Curvature–induced term}} \, ds \Bigg] \notag \\
    &+\;\frac{\eta}{2}\,
    \mathbb{E}\!\Bigg[ \int_0^{T} \underbrace{ \operatorname{Tr}\!\Big(\nabla^2_\bTheta f(x, \mathbf{Z}_s)\;\bs_{|\mathcal{B}|}(\mathbf{Z}_s)\Big) }_{\text{Sampling-induced term}} \, ds \Bigg] + O(\eta^2). \label{Domingo sgd2}
\end{align}
\end{theorem}

\begin{proof}
     For an input $x$, the expectation of the output after $k$ iterations can be written as
    \begin{align}\label{expectation2_sgd}
         \mathbb{E}[f(x,\bTheta_{K})]&=\mathbb{E}[f(x,\bZ_{T})]+\mathbb{E}[f(x,\bTheta_{K})]-\mathbb{E}[f(x,\bZ_{T})]\notag\\
        &=\mathbb{E}[f(x,\bZ_{T})]+O(\eta^2).
    \end{align}
    The last equation comes from Lemma \ref{2 order}.
    Since $\bZ_{T}$ is a continuous It\^{o} process, the It\^{o} formula gives
    \begin{align*}
        f(x,\bZ_{T})
        =& f(x,\bZ_0)
          + \sum_{j=1}^d \int_0^{T}
            \frac{\partial f}{\partial \Theta^j}(x,\bZ_s)\,\mathrm{d}Z_s^j\notag\\
          &+ \frac{1}{2}\sum_{j,l=1}^d \int_0^{T}
            \frac{\partial^2 f}{\partial \Theta^j \partial \Theta^l}(x,\bZ_s)
            \,\mathrm{d}\langle Z^j, Z^l \rangle_s,
    \end{align*}
    where $Z_s^j$ denotes the $j$-th coordinate of $\bZ_s$.
    Denote by $\bigl(a_{jl}(\bZ)\bigr)_{j,l=1}$  the $(j,l)$-entry of $\bs_{|\cB|}^{1/2}(\bZ)$. From Eq.(\ref{eq:2 order}), we have
    \begin{align*}
        &f(x,\bZ_{T})=f(x,\bTheta_0)-\int_0^{T} \left(\nabla_{\bTheta} f(x,\bZ_s)\right)^\top\nabla_{\bTheta} L(\bZ_s)ds\\& -\frac{\eta}{2}\int_0^{T} \left(\nabla_{\bTheta} f(x,\bZ_s)\right)^\top\nabla^2_\bTheta L(\bZ_s)\nabla L(\bZ_s)ds+\sqrt{\eta}\int_0^{T}\left(\nabla_{\bTheta} f(x,\bZ_s)\right)^\top\bs_{|\mathcal{B}|}^\frac{1}{2}(\bZ_s)d\bW_s\\&+\frac{\eta}{2}\sum_{j,l=1}^d \int_0^{T}\frac{\partial^2 f}{\partial \bTheta^j \partial \bTheta^l}(x,\bZ_s)\sum_{h=1}^d a_{jh}(\bZ_s)a_{lh}(\bZ_s)ds\\
        \end{align*}
        \begin{align*}
        =&f(x,\bTheta_0)-\int_0^{T} \left(\nabla_{\bTheta} f(x,\bZ_s)\right)^\top\nabla_{\bTheta} L(\bZ_s)ds\\&-\frac{\eta}{2}\int_0^{T} \left(\nabla_{\bTheta} f(x,\bZ_s)\right)^\top \nabla^2_\bTheta L(\bZ_s)\nabla L(\bZ_s)ds\\&+\sqrt{\eta}\int_0^{T}\,\,\left(\nabla_{\bTheta} f(x,\bZ_s)\right)^\top \bs^\frac{1}{2}_{|\mathcal{B}|}(\bZ_s)d\bW_s+\frac{\eta}{2}\int_0^{T}\mathrm{Tr}\left(\nabla_{\bTheta}^2 f(x,\bZ_s)\bs_{|\mathcal{B}|}(\bZ_s)\right)ds.\qquad\qquad
   \end{align*}
   Expressing $\nabla_\bTheta L(\bZ_s)$ in terms of the empirical average over samples $\{(x_n,y_n^*)\}_{n=1}^N$, we can rewrite the above result as
    \begin{align*}
        &f(x,\bZ_{T})=f(x,\bTheta_0)-\frac{1}{N}\sum_{n=1}^N\int_0^{T} \big\langle \nabla_\bTheta f(x,\bZ_s),\nabla_\bTheta f(x_n,\bZ_s)\big\rangle \frac{\partial \ell}{\partial f}(f(x_n,\bZ_s),y_n^*)ds\\&-\frac{\eta}{2}\int_0^{T} \left(\nabla_{\bTheta} f(x,\bZ_s)\right)^\top\nabla^2_\bTheta L(\bZ_s)\frac{1}{N}\sum_{n=1}^N(\nabla_{\bTheta} f(x_n,\bZ_s))\frac{\partial \ell}{\partial f}(f(x_n,\bZ_s),y_n^*))ds\\
    &+\sqrt{\eta}\int_0^{T}\left(\nabla_{\bTheta} f(x,\bZ_s)\right)^\top\bs^\frac{1}{2}_{|\mathcal{B}|}(\bZ_s)d\bW_s+\frac{\eta}{2}\int_0^{T} \mathrm{Tr}\left(\nabla_{\bTheta}^2 f(x,\bZ_s)\bs_{|\mathcal{B}|}(\bZ_s)\right)ds\\
    &=f(x,\bTheta_0)-\frac{1}{N}\sum_{n=1}^N\int_0^{T} K_s(x,x_n)\frac{\partial \ell}{\partial f}(f(x_n,\bZ_s),y_n^*)ds\\
    &+\sqrt{\eta}\int_0^{T}\left(\nabla_{\bTheta} f(x,\bZ_s)\right)^\top\bs^\frac{1}{2}_{|\mathcal{B}|}(\bZ_s)d\bW_s+\frac{\eta}{2}\int_0^{T} \mathrm{Tr}\left(\nabla_{\bTheta}^2 f(x,\bZ_s)\bs_{|\mathcal{B}|}(\bZ_s)\right)ds\\
    &-\frac{\eta}{2N}\sum_{n=1}^N\int_0^{T} K_s^{\text{cur}}(x,x_n)\frac{\partial \ell}{\partial f}(f(x_n,\bZ_s),y_n^*)ds.
    \end{align*}
    Plugging the above equation into (\ref{expectation2_sgd}) and noting that $$\mathbb{E}\left[\int_0^{T} \left(\nabla_\bTheta f(x,\bZ_s)\right)^\top\bs_{|\mathcal{B}|}^\frac{1}{2}(\bZ_s)d\bW_s\right]=0,$$ we obtain the desired result.
\end{proof}

 Equation~\eqref{Domingo sgd2} decomposes the expected prediction into three parts.
     \begin{itemize}
         \item \textit{Path-kernel term.} Along the training path $(\bZ_s)_{s\ge0}$, the initial prediction $f(x,\bTheta_0)$ is modified by an accumulated contribution from all training samples: each sample's residual $\frac{\partial \ell}{\partial f}(f(x_n,\bZ_s),y_n^*)$ is weighted by the path kernel.
\item \textit{Curvature-induced term.} Geometrically, the curvature-weighted path kernel extends the standard path kernel by incorporating the local geometry of the loss landscape. This term retains an interpolation structure, but now along random training trajectories. Moreover, since it carries a factor of $\eta$, it represents a second-order correction to the leading path-kernel term rather than a dominant contribution. A large magnitude of $K_s^{\mathrm{cur}}(x,x_n)$ indicates strong alignment between the tangent features of the test and training points through the local Hessian geometry of the loss; the corresponding training residual then exerts a larger second-order effect on the prediction.
\item \textit{Sampling-induced term.} This noise-induced term is specific to SGD and depends on the mini-batch sampling noise. From a geometric perspective, $\bs_{|\mathcal B|}(\bZ_s)$ is the covariance matrix of the mini-batch gradient noise at the parameter $\bZ_s$, and $\nabla_\bTheta^2 f(x,\bZ_s)$ quantifies the local curvature of the predictor at $x$. The term thus captures the interaction between mini-batch randomness and the curvature of the model output. Since smaller batch sizes yield a larger covariance, the effect of sampling-induced noise becomes stronger as the batch size decreases. This contribution is absent in the deterministic case.
\end{itemize}
Theorem~\ref{2-order Domingo} extends Domingos' theorem to the setting of a second-order weak approximation of SGD. Beyond the first-order characterization in terms of the weighted path kernels, the result introduces additional curvature-dependent terms that capture the influence of the loss Hessian and the second-order derivatives of the model. In particular, the sampling-induced term reflects the interaction between the noise covariance and the Hessian of the model output, highlighting the role of mini-batch sampling in shaping the optimization trajectory. Consequently, the trained model can be viewed as a correction of the initialization $f(x,\bTheta_0)$ by the training residuals, weighted by path kernels and curvature-induced kernels, where these weights encode the subtle effects of loss-landscape geometry and gradient noise in the SGD dynamics.

\subsection{Second-order transported path kernel representation for SGDM}
In the presence of momentum, the learning dynamics acquire a temporal memory effect, in which past gradients continue to influence the update with exponentially decaying weight. This mechanism reshapes how training examples contribute during optimization, leading to a modified first-order form of the model's output under SGDM. To capture the finer structure of momentum dynamics under stochastic perturbations, we extend Domingos' theorem via a second-order weak approximation of SGDM. The following theorem provides a precise formulation of this result.

\begin{lemma}[Second-order weak approximation of SGDM~\cite{li2019stochastic,wang2023marginal}]\label{2 order SGDM}
Let $T>0$, $\eta\in (0,1\wedge T)$, and set $K=\lfloor T/\eta\rfloor$. Assume the same conditions as in Lemma~\ref{2 order} and the drift terms in the following SDEs are Lipschitz. Let $\mu=\frac{1-\beta}{\eta}$ be fixed, and define $\bM_t,\bZ_t:t\in [0, T]$ as the stochastic process satisfying the SDEs
    \begin{align}\label{2-order sde SGDM}
         d\bM_t =& -\left[\left(\mu \bI+\frac{\eta}{2}[\mu^2 \bI-\nabla^2_\bTheta L(\bZ_t)]\right)\bM_t+\left(1+\frac{\eta}{2}\mu\right)\nabla_\bTheta L(\bZ_t)\right]dt\\&+\sqrt{\eta}\,\bs_{|\mathcal{B}|}^\frac{1}{2}(\bZ_t)\,d\bW_t, \qquad \bM_0=0,\notag\\
         d\bZ_t =& \left[\left(1-\frac{\eta}{2}\mu\right) \bM_t-\frac{\eta}{2}\nabla_\bTheta L(\bZ_t)\right]dt, \qquad \bZ_0=\bTheta_0,
    \end{align}
    where $\bI$ is the identity matrix and $\bs_{|\mathcal{B}|}(\bZ_t)$ is defined in Lemma~\ref{2 order}. Then $\{\bZ_t:t\in [0, k\eta]\}$ is a second-order weak approximation of SGDM.
\end{lemma}
Lemma~\ref{2 order SGDM} provides the continuous-time second-order approximation of the SGDM iterates. We next use this representation to derive the corresponding expansion for the expected model output.

\begin{theorem}[Stochastic Domingos' theorem (SGDM)]\label{thm:SGDM_second_order}
Consider a learning model \(y=f(x,\bTheta)\) with \(f(x,\cdot)\in \mathcal G^3\). Suppose the parameter \(\bTheta\) is learned from a dataset \(\{(x_n,y_n^*)\}_{n=1}^N\) using stochastic gradient descent with momentum (SGDM), and $\nabla_\bTheta^2 L(\bZ_t)$ is bounded on $[0,T]$. Assume the same conditions as in Theorem~\ref{2-order Domingo}, and the boundedness conditions in Lemma~\ref{lem:stochastic-convolution-bound} hold along the trajectory. Then, in the sense of a second-order weak approximation of SGDM, we have
\begin{equation}\label{eq:domingos_sgdm}
\begin{aligned}
\mathbb E[f(x,\bTheta_K)]
=&\,f(x,\bTheta_0)-\frac1N\sum_{n=1}^N\mathbb E\,\bigg[
\int_0^{T}\underbrace{\int_0^t K_{t,s}(x,x_n)\frac{\partial \ell}{\partial f}\bigl(f(x_n,\bZ_s), y_n^*\bigr)\,ds}_{\text{memory-weighted path kernel}}\,dt
\bigg]\\
&-\frac{\eta}{2N} \sum_{n=1}^N\mathbb E\,\bigg[\int_0^{T}\underbrace{\int_0^t K^{\mathrm{cur}}_{t,s}(x,x_n)\frac{\partial \ell}{\partial f}\bigl(f(x_n,\bZ_s), y_n^*\bigr)\,ds}_{\text{accumulated curvature-induced term}}\,dt\bigg]\\
&-\frac{\eta}{2N}\sum_{n=1}^N\,\mathbb E\,\bigg[\int_0^{T}\underbrace{ K_t(x,x_n)\frac{\partial \ell}{\partial f}\bigl(f(x_n,\bZ_t), y_n^*\bigr)}_{\text{path-kernel term}}\,dt
\bigg]\\
&+\sqrt{\eta}\,\mathbb E\,\bigg[\int_0^{T}\underbrace{\int_0^t
e^{-\mu(t-s)}\operatorname{Tr}\!\Big((D_s\bZ_t)^\top\nabla_\bTheta^2 f(x,\bZ_t)\,\bs_{|\mathcal B|}^{1/2}(\bZ_s)\Big)\,ds}_{\text{transported sampling-induced term}}\,dt\bigg]
\\&+O(\eta^2),
\end{aligned}
\end{equation}
where \(\mu=\frac{1-\beta}{\eta}\) is fixed, and
\begin{align}
&K_{t,s}(x,x_n):=\left\langle \nabla_\bTheta f(x,\bZ_t),\,e^{-\mu(t-s)} \left(1-\frac{\mu^2}{2}(t-s)\eta\right)\nabla_\bTheta f(x_n,\bZ_s) \right\rangle,
\label{eq:def_memory_kernel_final}\\
&K_{t,s}^{\mathrm{cur}}(x,x_n):=\left\langle\nabla_\bTheta f(x,\bZ_t),\,e^{-\mu(t-s)}\nabla_\bTheta f(x_n,\bZ_s)
\right\rangle_{\int_s^t \nabla_\bTheta^2L(\bZ_r)\,dr},
\label{eq:def_cur_kernel_final}\\
& K_t(x,x_n):=\left\langle\nabla_\bTheta f(x,\bZ_t),\,\nabla_\bTheta f(x_n,\bZ_t) \right\rangle.
\end{align}
Here, \(D_s\) denotes the Malliavin derivative~\cite[Definition 1.2.1]{nualart2006malliavin} with respect to the driving Brownian motion at time \(s\).
\end{theorem}

\begin{proof}
    For any input $x$, the expectation of the output in the $k$-th iteration is
    \begin{align}\label{expectation2}
         \mathbb{E}[f(x,\bTheta_{K})]&=\mathbb{E}[f(x,\bZ_{T})]+\mathbb{E}[f(x,\bTheta_{K})]-\mathbb{E}[f(x,\bZ_{T})]\notag\\
        &=\mathbb{E}[f(x,\bZ_{T})]+O(\eta^2).
    \end{align}
    Denote $A_t:=\mu \mathbf{I}+\frac{\eta}{2}\bigl[\mu^2 \mathbf{I}-\nabla_\bTheta^2L(\bZ_t)\bigr]$, then the SDE~\eqref{2-order sde SGDM} for $\bM_t$ in Lemma~\ref{2 order SGDM} can be written as
\begin{equation}\label{eq:plugdM}
\mathrm{d}\bM_t
=-\Bigl[A_t\bM_t+\Bigl(1+\frac{\eta}{2}\mu\Bigr)\nabla_\bTheta L(\bZ_t)\Bigr]dt+\sqrt{\eta}\,\bs_{|B|}^{1/2}(\bZ_t)\,\mathrm{d}\bW_t.
\end{equation}

For any $t,s$ with $t\geq s\geq 0$, let $\Phi_{t,s}$ be the fundamental matrix associated with the linear system
\begin{align}\label{eq:odePhi}
\partial_t \Phi_{t,s}=-A_t\Phi_{t,s},
\qquad
\Phi_{s,s}=\mathbf{I}.
\end{align}
Since $A_t$ is bounded along the trajectory, the standard theory of nonautonomous linear systems ensures the existence of the fundamental matrix $\Phi_{t,s}$, see~\cite[Corollary 2.3]{logemann2014ordinary}.

From $\Phi_{t,s}^{-1}\Phi_{t,s}=\mathbf{I}$, we obtain
$\partial_t\Phi_{t,s}^{-1}=\Phi_{t,s}^{-1}A_t.$ Moreover, we note that the fundamental matrix satisfies
\begin{equation}\label{evolution_property}
    \Phi_{t,s}=\Phi_{t,0}\Phi_{s,0}^{-1},\,\,0\le s\le t.
\end{equation}
Taking the differential of $\Phi_{t,0}^{-1}\bM_t$ and using Eq.\eqref{eq:plugdM} we have
    \begin{align}\label{eq:phi_M}
    \begin{aligned}
d(\Phi_{t,0}^{-1}\bM_t)&=-\Bigl(1+\frac{\eta}{2}\mu\Bigr)\Phi_{t,0}^{-1}\nabla_\bTheta L(\bZ_t)dt+\sqrt{\eta}\,\Phi_{t,0}^{-1}\bs_{|B|}^{1/2}(\bZ_t)\,\mathrm{d}\bW_t.
        \end{aligned}
    \end{align}
Integrating both sides of Eq.(\ref{eq:phi_M}),  we have
\begin{equation}\label{eq:phi_M_int}
    \int_0^t d(\Phi_{s,0}^{-1}\bM_s)=-\Bigl(1+\frac{\eta}{2}\mu\Bigr)\int_0^t\Phi_{s,0}^{-1}\nabla_\bTheta L(\bZ_s)ds+\sqrt{\eta}\int_0^t\,\Phi_{s,0}^{-1}\bs_{|B|}^{1/2}(\bZ_s)\,\mathrm{d}\bW_s.
\end{equation}
Multiplying both sides of Eq.(\ref{eq:phi_M_int}) by $\Phi_{t,0}$, and combining with Eq.(\ref{evolution_property}),  we obtain, since $\bM_0=0$,
\begin{equation}\label{eq:Mt_phi_representation}
\bM_t=-\int_0^t
\Phi_{t,s}\Bigl(1+\frac{\eta}{2}\mu\Bigr)\nabla_\bTheta L(\bZ_s)\,ds+\sqrt{\eta}\int_0^t\Phi_{t,s}\bs_{|\mathcal{B}|}^{1/2}(\bZ_s)\,d\bW_s.
\end{equation}
 By Eq.(\ref{2-order sde SGDM}) in Lemma~\ref{2 order SGDM},  
 \begin{equation*}
     d\bZ_t=\left[\left(1-\frac{\eta}{2}\mu\right)\bM_t-\frac{\eta}{2}\nabla_\bTheta L(\bZ_t)\right]dt.
 \end{equation*}
Since $\bZ_t$ is a continuous finite-variation process, we have $d\langle \bZ^j,\bZ^\ell\rangle_t=0$. Hence, by chain rule we obtain
\begin{equation}
\begin{aligned}
   \label{eq:dfdt_Zt_start}
\frac{d}{dt}f(x,\bZ_t)
&=\nabla_\bTheta f(x,\bZ_t)^\top \frac{d\bZ_t}{dt}\\
&=\nabla_\bTheta f(x,\bZ_t)^\top
\left(\Bigl(1-\frac{\eta}{2}\mu\Bigr)\bM_t-\frac{\eta}{2}\nabla_\bTheta L(\bZ_t)\right). 
\end{aligned}
\end{equation}

Substituting Eq.\eqref{eq:Mt_phi_representation}  into Eq.\eqref{eq:dfdt_Zt_start} and reorganizing terms gives
\begin{equation}\label{eq:dfdt_Zt_decomp}
\frac{d}{dt}f(x,\bZ_t)
=
\mathcal I_t^{(1)}(x)+\mathcal I^{(2)}_t(x)+\mathcal I^{(3)}_t(x)+O(\eta^2),
\end{equation}
where
\begin{align}
\mathcal I_t^{(1)}(x)
&:=-\left(1-\frac{\eta}{2}\mu\right)\left(1+\frac{\eta}{2}\mu\right)
\left\langle \nabla_{\bTheta}f(x,\bZ_t), \int_0^t \Phi_{t,s}\,\nabla_{\bTheta}L(\bZ_s)\,ds \right\rangle, \label{eq:def_Pt}
\\
\mathcal I_t^{(2)}(x)
&:=-\frac{\eta}{2}\left\langle
\nabla_{\bTheta}f(x,\bZ_t),\nabla_{\bTheta}L(\bZ_t)
\right\rangle,
\\
\mathcal I_t^{(3)}(x)
&:=\left(1-\frac{\eta}{2}\mu\right)\sqrt{\eta}\,
\left\langle
\nabla_{\bTheta}f(x,\bZ_t),
\int_0^t \Phi_{t,s}\,\bs_{|B|}^{1/2}(\bZ_s)\,d\bW_s
\right\rangle.
\end{align}

From Lemma~\ref{lemma:phi}, we have
\begin{equation}\label{eq:phi_appro}
   \Phi_{t,s}=e^{-\mu(t-s)}\left(\mathbf{I}+\frac{\eta}{2}\int_s^t \nabla_\bTheta^2L(\bZ_r)dr-\frac{\mu^2}{2}(t-s)\eta\,\mathbf{I}\right)+O(\eta^2). 
\end{equation}
Plugging Eq.(\ref{eq:phi_appro}) into Eq.(\ref{eq:def_Pt}), since $\left(1-\frac{\eta}{2}\mu\right)\left(1+\frac{\eta}{2}\mu\right)=1+O(\eta^2)$, we obtain
\begin{equation}\label{eq:I1_expand}
\begin{aligned}
\mathcal I_t^{(1)}(x)
=&-
\left\langle
\nabla_{\bTheta}f(x,\bZ_t),
\int_0^t \Phi_{t,s}\,\nabla_{\bTheta}L(\bZ_s)\,ds
\right\rangle+O(\eta^2)\\
=&-\frac{1}{N}\sum_{n=1}^N \int_0^t
e^{-\mu(t-s)}
\overline{K}_{t,s}(x,x_n)
\frac{\partial \ell}{\partial f}\bigl(f(x_n,\bZ_s),y_n^*\bigr)\,ds
\\
&-\frac{\eta}{2N}\sum_{n=1}^N\int_0^te^{-\mu(t-s)}\overline{K}^{\text{cur}}_{t,s}(x,x_n)\frac{\partial \ell}{\partial f}\bigl(f(x_n,\bZ_s),y_n^*\bigr)\,ds\\
&+\frac{\mu^2}{2N}\eta\sum_{n=1}^N\int_0^t (t-s)e^{-\mu(t-s)}\overline{K}_{t,s}(x,x_n)\frac{\partial \ell}{\partial f}\bigl(f(x_n,\bZ_s),y_n^*\bigr)\,ds
+O(\eta^2).
\end{aligned}
\end{equation}
where \begin{align*}
    &\overline{K}_{t,s}(x,x_n):=\left\langle\nabla_{\!\bTheta} f(x,\bZ_t),\nabla_{\!\bTheta} f(x_n,\bZ_s)\right\rangle,\\ &\overline{K}_{t,s}^{\text{cur}}(x,x_n):=\left\langle\nabla_{\!\bTheta} f(x,\bZ_t),\nabla_{\!\bTheta} f(x_n,\bZ_s)\right\rangle_{\int_s^t \nabla^2_\bTheta L(\bZ_r)dr}.
\end{align*}
With the definition of $K_t(x,x')$ (Eq.(\ref{eq_gradient_kernel})), we can also rewrite
\begin{equation}\label{eq:I2_expand}
\mathcal I_t^{(2)}(x)
=
-\frac{\eta}{2N}\sum_{n=1}^N
K_t(x,x_n)
\frac{\partial \ell}{\partial f}\bigl(f(x_n,\bZ_t),y_n^*\bigr).
\end{equation}

Now we expand \(\Phi_{t,s}\) inside \(\mathcal I_t^{(3)}\). Invoking Lemma~\ref{lemma:phi} again, we obtain
\begin{equation}\label{eq:Nt_expand}
\begin{aligned}
\mathcal I_t^{(3)}(x)
=
&\int_0^t
e^{-\mu(t-s)}
\left[
\left(1-\frac{\eta}{2}\mu\right)\sqrt{\eta}
-\frac{\mu^2}{2}\eta^{3/2}(t-s)
\right]
\nabla_{\bTheta}f(x,\bZ_t)^\top
\bs_{|B|}^{1/2}(\bZ_s)\,d\bW_s
\\
&+\frac{\eta\sqrt{\eta}}{2}
\int_0^t e^{-\mu(t-s)}
\nabla_{\bTheta}f(x,\bZ_t)^\top
\left(\int_s^t \nabla_{\bTheta}^2L(\bZ_r)\,dr\right)
\bs_{|B|}^{1/2}(\bZ_s)\,d\bW_s
+O(\eta^2).
\end{aligned}
\end{equation}
By the same argument as in Lemma~\ref{lem:stochastic-convolution-bound}, after taking expectation, each stochastic integral in~\eqref{eq:Nt_expand} is of order \(O(\sqrt{\eta})\), hence the two terms carrying the prefactor \(\eta^{3/2}\)  can be absorbed into the remainder. Thereby,
\begin{align}\label{eq:I_3malliavin}
   \mathbb E\left[\int_0^{T} \mathcal I_t^{(3)}(x)dt\right]=
&\sqrt{\eta}\mathbb E\left[\int_0^{T}\nabla_{\bTheta}f(x,\bZ_t)^\top\int_0^t
e^{-\mu(t-s)}\bs_{|B|}^{1/2}(\bZ_s)\,d\bW_sdt\right]+O(\eta^2).
\end{align}
The expectation in Eq.(\ref{eq:I_3malliavin}) does not vanish in general, since its integrand contains the future quantity \(\nabla_{\bTheta} f(x,\bZ_t)\) and is therefore not adapted to \(\mathcal F_s\). To identify its contribution, we use the Malliavin duality formula~\cite[Definition 1.3.1]{nualart2006malliavin}; see also \cite[Lemma 2.1]{debussche2011weak}. It yields that
\begin{align*}
    &\mathbb E\!\left[\int_0^{T}
\nabla_\bTheta f(x,\bZ_t)^\top
\int_0^t e^{-\mu(t-s)} \bs_{|\mathcal B|}^{1/2}(\bZ_s)\,d\bW_sdt
\right]\\
&=\mathbb E\!\left[\int_0^{T}
\int_0^t e^{-\mu(t-s)}
\operatorname{Tr}\!\Big(
D_s(\nabla_\bTheta f(x,\bZ_t))^\top\,
\bs_{|\mathcal{B}|}^{1/2}(\bZ_s)
\Big)\,dsdt
\right],
\end{align*}
where \(D_s\) denotes the Malliavin derivative with respect to the Brownian motion at time \(s\).
Using the chain rule \cite[Proposition 1.2.3]{nualart2006malliavin},
\[
D_s(\nabla_\bTheta f(x,\bZ_t))
=\nabla_\bTheta^2 f(x,\bZ_t)D_s\bZ_t ,
\]
 Eq.(\ref{eq:I_3malliavin}) then becomes
\begin{align}\label{eq:I3_expand}
    \mathbb E\left[\int_0^{T} \mathcal I_t^{(3)}(x)dt\right]
=&
\sqrt{\eta}\mathbb E\!\left[\int_0^{T}
\int_0^t e^{-\mu(t-s)}
\operatorname{Tr}\!\Big(
(D_s\bZ_t)^\top \nabla_\bTheta^2 f(x,\bZ_t)\,
\bs_{|\mathcal B|}^{1/2}(\bZ_s)
\Big)\,dsdt
\right]\notag\\&+O(\eta^2).
\end{align}
Integrating Eq.\eqref{eq:dfdt_Zt_decomp} over $[0,T]$ and taking expectations, we thus have
\begin{align}
    \mathbb{E}[f(x,\bZ_{T})]-f(x,\bZ_0)=\mathbb{E}\left[\int_0^{T}\left(I_t^{(1)}(x)+I_t^{(2)}(x)+I_t^{(3)}(x)\right)\,dt\right]+O(\eta^2).
\end{align}
Plugging \eqref{eq:I1_expand}, \eqref{eq:I2_expand} and \eqref{eq:I3_expand} into the above equation, and using Eq.(\ref{expectation2}) yields the desired result.
\end{proof}

Equation~\eqref{eq:domingos_sgdm} decomposes the expected prediction under SGDM into four contributions.
\begin{itemize}
    \item \textit{Memory-weighted path kernel.} This is the leading deterministic memory contribution in the SGDM representation. It shows that SGDM replaces the instantaneous path kernel with a time-transported kernel: a residual signal generated by sample \(x_n\) at time \(s\) is propagated to a later time \(t\) with weight \(e^{-\mu(t-s)}\bigl(1-\frac{\mu^2}{2}(t-s)\eta\bigr)\), and contributes to the prediction at the test point through the alignment of tangent features.
    \item \textit{Accumulated curvature-induced term.} This term describes how loss curvature modifies the transport of past residual signals. A training residual generated at time \(s\) is first propagated forward by the exponential memory factor \(e^{-\mu(t-s)}\) and then reweighted by the accumulated Hessian along the path from \(s\) to \(t\). Curvature thus changes how past training information is retained and transported before it affects the tangent-feature comparison at the test point.
    \item \textit{Path-kernel term.} This term retains the ordinary path-kernel structure, but with a coefficient proportional to \(\eta\). It measures the instantaneous alignment between the tangent feature of the test point and that of each training sample at the same time \(t\).
   \item \textit{Transported sampling-induced term.} This term shows that mini-batch sampling noise injected at time \(s\) is carried forward to time \(t\) through the momentum dynamics with exponential weight \(e^{-\mu(t-s)}\). Its contribution to the prediction is then measured by the tangent feature at the test point.
\end{itemize}

\section{Implications of the Second-order Theory}\label{inplications}
In this section, we discuss some implications of our second-order theory for the model's predictions.
\subsection{Curvature-weighted interpolation}
For the common case of a quadratic loss \(\ell(f,y^*)=(f-y^*)^2\), the derivative \(\partial\ell/\partial f\) equals \(2\bigl(f(x_n,\bTheta_t)-y_n^*\bigr)\). Hence Eq.~\eqref{continuous Domingos} reduces to an interpolation formula:
\begin{align}\label{simpleDomingo}
y(T) - y(0) =& -\frac{2}{N}\sum_{n=1}^N 
\int_0^T 
\bigl\langle \nabla_{\!\bTheta} f(x,\bTheta_t), \nabla_{\!\bTheta} f(x_n,\bTheta_t) \bigr\rangle \bigl(f(x_n,\bTheta_t) - y_n^*\bigr)\, dt\notag\\
&-\frac{\eta}{N}\sum_{n=1}^N \int_0^T \langle\nabla_\bTheta f(x,\bTheta_t),\nabla_\bTheta f(x_n,\bTheta_t)\rangle_{\nabla^2_\bTheta L(\bTheta_t)}\bigl(f(x_n,\bTheta_t) - y_n^*\bigr)\,dt\\&+O(\eta^2)\notag,
\end{align}
for any \(x\), where \(y(T):=f(x,\bTheta_T)\) and \(y(0):=f(x,\bTheta_0)\). Eq.~\eqref{simpleDomingo} shows that \(y(T)-y(0)\) is a weighted sum of all training-point prediction corrections \(f(x_n,\bTheta_t)-y_n^*\), with weights given by the gradient kernel integrated over the training path. In the leading term, these corrections are weighted by the path kernel. The second term has the same interpolation structure, but with a curvature-weighted kernel. Consequently, the effect of a training sample on the output depends not only on tangent-feature alignment, but also on how this alignment interacts with the local geometry of the loss landscape.

\subsection{Concentration of model predictions}
Theorem~\ref{2-order Domingo} describes the mean prediction of models trained with SGD. In fact, we can say more about the variability of the prediction: it is governed jointly by the learning rate, the covariance of the sampling-induced noise, and the local geometry of the loss landscape.

\begin{proposition}[Concentration of the terminal prediction]\label{prop:terminal-prediction-concentration}
Let \(T=K\eta\), and let \(\bZ_t\) be the stochastic process defined in Lemma~\ref{2 order}. Let \(\mathcal U\subset\mathbb R^d\) be an open region, and define the exit time $\tau:=\inf\{t\ge 0: \bZ_t\notin \mathcal U\}$.
Assume that, for a fixed test point $x$, there exist constants $G_L, M_2, M_3, K_\Sigma>0$ and batch-size-dependent constants $S_{|\mathcal{B}|}>0$ and $G_f(x)>0$ such that, for every $z\in \mathcal U$,
\[
\|\nabla L(z)\|\le G_L,\quad
\|\nabla^2L(z)\|_{\mathrm{op}}\le M_2,\quad
\|\nabla^3L(z)\|_{\mathrm{op}}\le M_3,\quad
\|\bs_{|\mathcal{B}|}(z)\|_{\mathrm{op}}\le S^2_{|\mathcal{B}|},
\]
and $\|\nabla_\bTheta f(x,z)\|\le G_f(x)$. Assume further that the diffusion coefficient \(\bs_{|\mathcal{B}|}^{1/2}\) satisfies
\[
\sum_{q=1}^d
\left\|
D(\bs^{1/2}_{|\mathcal{B}|})_q(z)\,v
\right\|^2
\le
K_\Sigma^2\|v\|^2,
\qquad
\forall z\in\mathcal U,\ \forall v\in\mathbb R^d,
\]
where \((\bs^{1/2}_{|\mathcal{B}|})_q(z)\) denotes the \(q\)-th column of
\(\bs^{1/2}_{|\mathcal{B}|}(z)\). Set
\(
\widetilde{K}_\eta:=M_2+\frac{\eta}{2}(M_3G_L+M_2^2)+\frac{\eta}{2}K_\Sigma^2 .
\)
Then, for every \(r>0\),
\begin{equation}
\mathbb P
\left(
\left|
f(x,\bZ_{T\wedge\tau})-\mathbb E[f(x,\bZ_{T\wedge\tau})]
\right|
\ge r
\right)
\le
2\exp\left(
-\frac{
r^2\widetilde{K}_\eta
}{
\eta S^2_{|\mathcal{B}|}G^2_f(x)
\left(e^{2\widetilde{K}_\eta T}-1\right)
}
\right).
\end{equation}
\end{proposition}
\begin{proof}
Since $f(x,\bZ_t)$ contains a nonzero drift term, it is not a martingale. To construct a martingale whose terminal value is the stopped terminal prediction
\(f(x,\bZ_{T\wedge\tau})\), we consider the same stochastic dynamics with initial condition specified at an arbitrary pair $(t,z)$. For $0\le t\le s\le T$ and $z\in \mathcal{U}$, let \(\bZ_s^{t,z}\) denote the process evolving according to the same dynamics as \(\bZ_s\), but initialized at the intermediate time \(t\) from the state \(z\), i.e. $\bZ_t^{t,z}=z$:
$$d\bZ_s^{t,z}
=\left[-\nabla_\bTheta L(\bZ_s^{t,z})-\frac{\eta}{2}\nabla^2_\bTheta L(\bZ_s^{t,z})\nabla_\bTheta L(\bZ_s^{t,z})\right]ds+\sqrt{\eta}\,\bs_{|\mathcal{B}|}^{1/2}(\bZ_s^{t,z})\,d\bW_s.$$
Let \(\tau^{t,z}:=\inf\{r\ge t: \bZ_r^{t,z}\notin \mathcal U\}\) be the corresponding exit time from
\(\mathcal U\), and define the stopped value function
\[
u_x(t,z):=\mathbb E\!\left[f\bigl(x,\bZ_{T\wedge\tau^{t,z}}^{t,z}\bigr)\right].
\]
By the Markov property, $u_x(t\wedge\tau,\bZ_{t\wedge\tau})
=\mathbb E\!\left[f(x,\bZ_{T\wedge\tau})
\mid \mathcal F_{t\wedge\tau}\right]$, so the process \(u_x(t\wedge\tau,\bZ_{t\wedge\tau})\) is a martingale with terminal value \(f(x,\bZ_{T\wedge\tau})\).
Moreover,
\[
u_x(0,\bTheta_0)
=\mathbb E\!\left[
f(x,\bZ_{T\wedge\tau})
\right],\qquad
u_x
\bigl(T\wedge\tau,\bZ_{T\wedge\tau}\bigr)
=
f(x,\bZ_{T\wedge\tau}).
\]
Hence
\begin{equation}\label{eq:m_t}
    M_t:=u_x(t\wedge\tau,\bZ_{t\wedge\tau})-u_x(0,\bZ_0)
\end{equation}
is a martingale satisfying $M_0=0$ and
$M_T=f(x,\bZ_{T\wedge\tau})-\mathbb E[f(x,\bZ_{T\wedge\tau})]$.

Before the exit time $\tau$, the stopped value function \(u_x\) satisfies the backward Kolmogorov equation~\cite[Theorem 8.1.1]{oksendal2013stochastic} on \(\mathcal U\),
\begin{equation}\label{eq:parti_u}
    \partial_t u_x(t,z)+\mathcal L_\eta u_x(t,z)=0,
\end{equation}
where, for all $\varphi\in \mathcal C^2(\mathbb{R}^d)$,
\[
\mathcal L_\eta\varphi(z)
=\left(-\nabla_{\!\bTheta}L(z)-\frac{\eta}{2}
\nabla_{\!\bTheta}^2L(z)\nabla_{\!\bTheta}L(z)\right)^\top
\nabla\varphi(z)+\frac{\eta}{2}
\operatorname{Tr}\!\left(
\bs_{|\mathcal B|}(z)\nabla^2\varphi(z)
\right).
\]
Applying It\^o's formula to
\(u_x(t,\bZ_t)\) up to the bounded stopping time \(\tau\wedge T\), we obtain
\begin{align}\label{eq:deri_u}
    \begin{aligned}
d u_x(t,\bZ_t)
=&\,\partial_t u_x(t,\bZ_t)\,dt+\nabla u_x(t,\bZ_t)^\top d\bZ_t
+\frac{1}{2}\operatorname{Tr}\!\left(\nabla^2 u_x(t,\bZ_t)\,d\langle \bZ\rangle_t
\right)\\
=&\,\Big[\partial_t u_x(t,\bZ_t)+\mathcal L_\eta u_x(t,\bZ_t)
\Big]dt+\sqrt{\eta}\,\nabla u_x(t,\bZ_t)^\top \bs_{|\mathcal B|}^{1/2}(\bZ_t)\,d\bW_t .
\end{aligned}
\end{align}
Substituting the backward Kolmogorov equation~\eqref{eq:parti_u} into~\eqref{eq:deri_u}, we have
\begin{equation*}
u_x(t\wedge\tau,\bZ_{t\wedge\tau})
-u_x(0,\bZ_0)=\sqrt{\eta}\int_0^{t\wedge\tau}\nabla u_x(s,\bZ_s)^\top \bs_{|\mathcal B|}^{1/2}(\bZ_s)\,d\bW_s .
\end{equation*}
Recall Eq.~\eqref{eq:m_t}, we get
\begin{equation}
    M_T=u_x(T\wedge\tau,\bZ_{T\wedge\tau})-u_x(0,\bZ_0)=\sqrt{\eta}\int_0^{T\wedge\tau}\nabla u_x(s,\bZ_s)^\top \bs_{|\mathcal B|}^{1/2}(\bZ_s)\,d\bW_s .
\end{equation}
Then the quadratic variation of $M$ is
\begin{equation}\label{eq:varia_MT}
    \langle M\rangle_T={\eta}\int_0^{T\wedge\tau}\nabla u_x(s,\bZ_s)^\top \bs_{|\mathcal B|}(\bZ_s)\,\nabla u_x(s,\bZ_s)\,ds .
\end{equation}
Since the process is stopped upon leaving \(\mathcal U\), all the bounds in the assumptions hold on the integration interval; hence the gradient estimate in Lemma~\ref{lem:backward-gradient-bound} also applies to the stopped value function \(u_x\). Then we obtain
\begin{align*}
    \nabla u_x(s,\bZ_s)^\top \bs_{|\mathcal B|}(\bZ_s)\,\nabla u_x(s,\bZ_s)&\le \|\bs_{|\mathcal B|}(\bZ_s)\|_{\mathrm{op}}\,\|\nabla u_x(s,\bZ_s)\|^2\\&\le S^2_{|\mathcal{B}|}G_f^2(x)\,e^{2\widetilde{K}_\eta (T-s)},
\end{align*}
and so Eq.(\ref{eq:varia_MT}) satisfies
\begin{equation*}
     \langle M\rangle_T\le \eta S^2_{|\mathcal{B}|}G_f^2(x)\int_0^T e^{2\widetilde{K}_\eta (T-s)}\,ds=\eta S^2_{|\mathcal{B}|}G_f^2(x)\,\frac{e^{2\widetilde{K}_\eta T}-1}{2\widetilde{K}_\eta}.
\end{equation*}
By the exponential martingale inequality,
\begin{equation*}
    \mathbb{P}(|M_T|\ge r)\le 2 \exp\!{\left(-\frac{r^2 \widetilde{K}_\eta}{\eta S^2_{|\mathcal{B}|}G_f^2(x)(e^{2\widetilde{K}_\eta T}-1)}\right)}.
\end{equation*}
Since $M_T=f(x,\bZ_{T\wedge\tau})-\mathbb E[f(x,\bZ_{T\wedge\tau})]$,
the desired inequality follows.
\end{proof}

Proposition~\ref{prop:terminal-prediction-concentration} provides a locally stability bound for the terminal prediction around its mean. Combined with the stochastic Domingos representation, it shows that the stopped prediction \(f(x,\bZ_{T\wedge\tau})\) concentrates around its expected kernel-type representation. On the event \(\{\tau>T\}\), this stopped prediction coincides with the original terminal prediction \(f(x,\bZ_T)\). The bound indicates that prediction stability is governed by the learning rate, the sampling-induced covariance, the curvature of the loss landscape, and the parameter sensitivity of the test prediction. The dependence on batch size enters through the covariance bound $S_{|\mathcal B|}^2$: larger batches strengthen concentration around the expected representation by reducing the sampling-induced covariance in the bound. The loss-curvature bounds $M_2$ and $M_3$ affect the concentration probability through $\widetilde K_\eta$. The factor \(G_f(x)\) further shows that test points whose predictions are more sensitive to parameter changes are less stable under stochastic training.

Denote by \(\mathcal A_2(x,T)\) the second-order approximation in Theorem~\ref{2-order Domingo}, excluding the \(O(\eta^2)\) remainder, so that 
\[
\mathbb E[f(x,\bZ_T)]=\mathcal A_2(x,T)+C\eta^2.
\]
For any \(\delta\in(0,1)\), define \[
R_\delta(x,T):=S_{|\mathcal B|}G_f(x) \sqrt{\eta\frac{e^{2\widetilde K_\eta T}-1}{\widetilde K_\eta}\log\frac{2}{\delta}}.\] From Proposition~\ref{prop:terminal-prediction-concentration}, we obtain the following  estimate:
\begin{equation}
    \mathbb P \Big(\left|f(x,Z_{T\wedge\tau})-\mathcal A_2(x,T)
\right|\le |C|\eta^2+R_\delta(x,T)
\Big)\ge 1-\delta.
\end{equation}
The parameter $\delta$ only enters through the logarithmic factor $\log(2/\delta)$, reflecting the usual trade-off between the size of the
error radius and the probability of the event.

\subsection{Sampling-induced variation and batch size effects} 

When the observable is the loss itself, then the  model loss  at $T$ is
\begin{align*}   
    \mathbb{E}[L(\bTheta_k)]=&\,
    L(\bTheta_0) 
    - \mathbb{E}\!\Bigg[ \int_0^{k\eta} \|\nabla_\bTheta L(\bZ_s)\|^2\, ds \Bigg]\\&-\frac{\eta}{2}\,
    \mathbb{E}\!\Bigg[ \int_0^{k\eta}  \nabla_\bTheta L(\bZ_s)^\top \nabla^2_\bTheta L(\bZ_s) \nabla_{\bTheta} L(\bZ_s) ds \Bigg] \\
    &+\;\frac{\eta}{2}\,
    \mathbb{E}\!\Bigg[ \int_0^{k\eta} \underbrace{\operatorname{Tr}\left(
\nabla_\bTheta^2 L(\mathbf Z_s)
\bs_{|\mathcal B|}(\mathbf Z_s)
\right)}_{\text{Sampling-induced term}}  \, ds \Bigg] + O(\eta^2). \label{Domingo loss sgd2}
\end{align*}

The representation above makes the role of mini-batch noise explicit: the sampling-induced term couples the local Hessian of the loss with the covariance of the mini-batch gradient noise. The batch-size dependence follows directly from the scaling of \(\bs_{|\mathcal B|}\). Larger batches exert weaker pressure against high-curvature regions and can therefore tolerate sharper minima. Conversely, small-batch SGD is strongly penalized in sharp regions and is thus naturally biased toward flatter ones.

\section{Experiments}\label{experiments}

\begin{figure}
    \centering
    \includegraphics[width=0.45\linewidth]{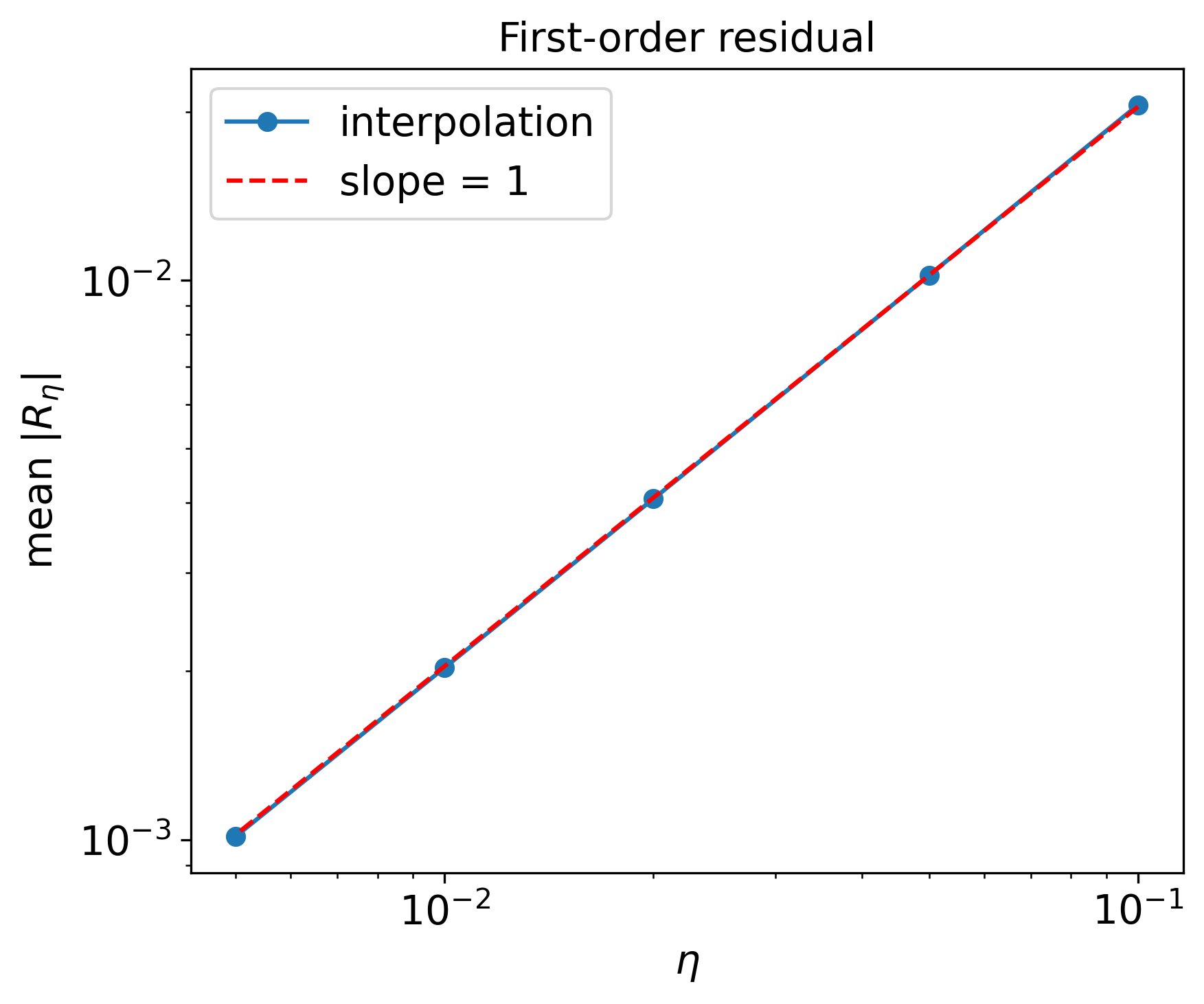}
    \includegraphics[width=0.45\linewidth]{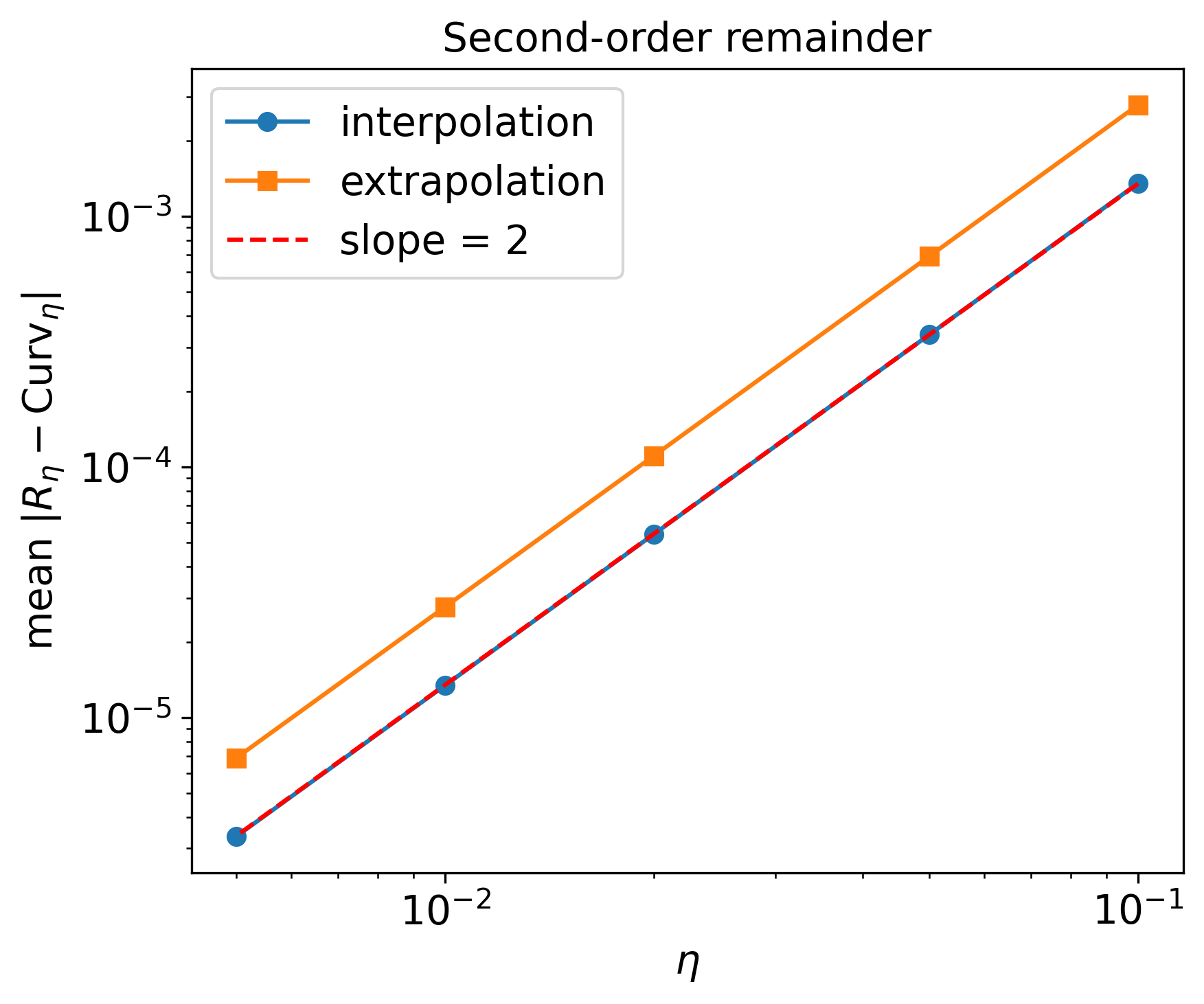}          
    \caption{Scaling verification for the output expansion under gradient descent. 
Left: the mean first-order residual $R_\eta$ over interpolation points scales as \(O(\eta)\) after subtracting the leading path-kernel term. 
Right: the mean second-order remainder $R_\eta(x)-\mathrm{Curv}_\eta(x)$ over interpolation points and extrapolation points scales as \(O(\eta^2)\) after further subtracting the curvature-induced term. 
Although the extrapolation error has a larger magnitude, it follows the same second-order scaling.}
    \label{fig:gd_val}
\end{figure}
In this section, we validate our second-order theory from different aspects via a series of numerical experiments.
\subsection{Scaling order verification}
We illustrate Theorem~\ref{thm_domingo} and Theorem~\ref{2-order Domingo} on a one-dimensional regression problem with target function
\begin{equation*}\label{eq:example_1d}
f_*(x)=\sin(2x)+0.4\sin(x).
\end{equation*}
A two-hidden-layer tanh network:
\begin{equation*}\label{eq:network_1d}
f(x,\bTheta) = a^{(2)}\sum_{m=1}^M\tanh(a_m^{(1)}x+b_m^{(1)}) + b^{(2)},
\end{equation*}
with learnable parameter $\bTheta=\{a_m^{(1)},b_{m}^{(1)}, m=1,\dots,M, a^{(2)}, b^{(2)}\}$ and width $M=8$ is trained by gradient descent on $N=20$ samples from 
$[-2,2]$, with terminal time $T=2.0$ and learning rates
\(\eta \in \{10^{-1},\,5\times10^{-2},\,2\times10^{-2},\,10^{-2},\,5\times10^{-3}\}\).
We evaluate the expansion Eq.(\ref{Domingo sgd2}) on interpolation points \(\{-1.22,-0.67,0.24,1.53\}\) and extrapolation points \(\{-3.5,2.5,2.7,3.0\}\).

For each test point \(x\), we compute the first-order residual
\[
R_\eta(x)
=f(x,\bTheta_T)-f(x,\bTheta_0)+\frac{1}{N}\sum_{n=1}^N\int_0^T
K_t(x,x_n)\,\frac{\partial \ell}{\partial f}\bigl(f(x_n,\bTheta_t),y_n^*\bigr)\,dt,\]
and the second-order remainder
\begin{align*}
    R_\eta-\text{Cur}_\eta:=&R_\eta(x)+\frac{\eta}{2N}\sum_{n=1}^N
     \int_0^{T} K_t^{\text{cur}}(x,x_n) \frac{\partial \ell}{\partial f}\bigl(f(x_n,\bZ_s), y_n^*\bigr)  \, dt .
\end{align*}
The time integrals are evaluated numerically along the discrete GD trajectory using trapezoidal quadrature.

The left panel of Figure~\ref{fig:gd_val} plots the mean interpolation residual and shows a slope \(1\) on a log-log scale, which confirms that the error after subtracting the leading path-kernel term is of order \(O(\eta)\). The right panel plots the mean second-order remainder for both interpolation and extrapolation points. In both cases, the curves exhibit slope \(2\), showing that after subtracting the curvature-induced term, the remaining error is of order \(O(\eta^2)\). While the extrapolation points have a larger error constant, the scaling law remains unchanged. This provides numerical evidence that the curvature-induced term is indeed the leading second-order contribution in the Theorem~\ref{thm_domingo}.

\begin{figure}
    \centering
    \includegraphics[width=0.9\linewidth]{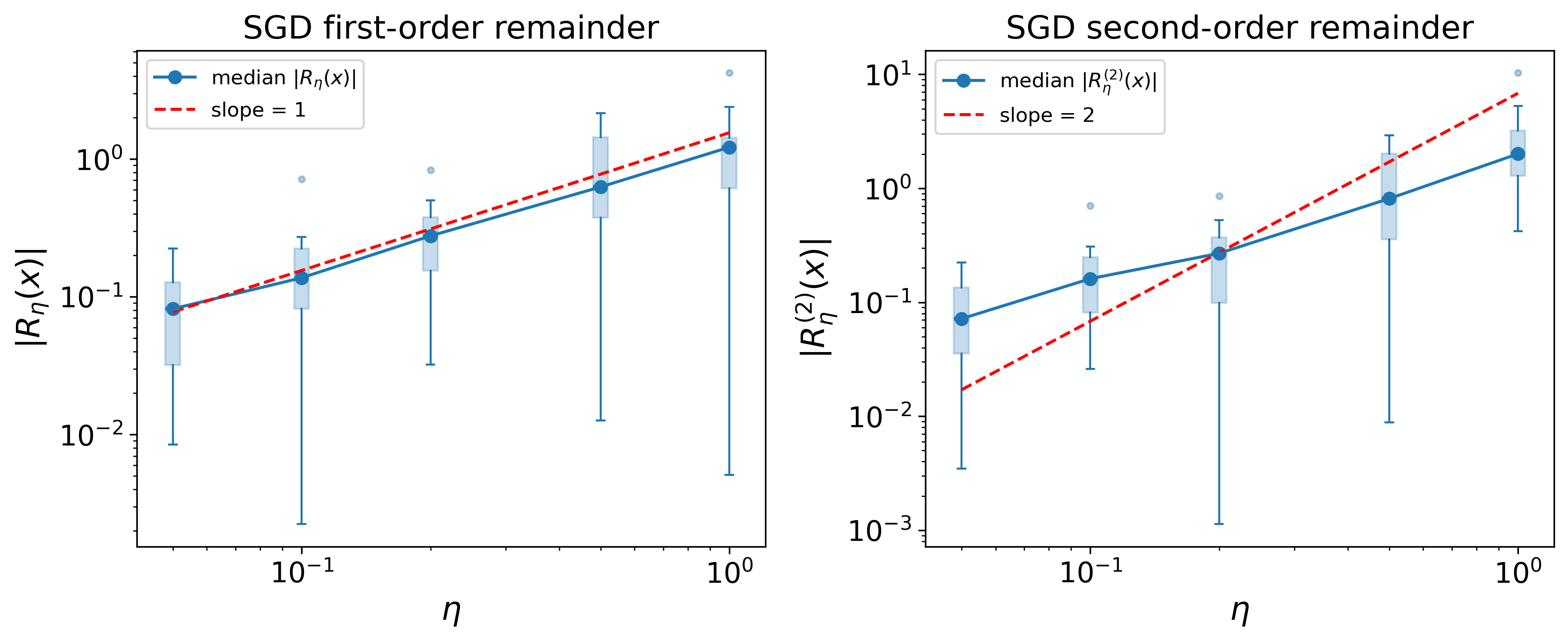}
    \caption{Empirical scaling of the first- and second-order remainders in the SGD output expansion at the test point $x=1.2$. For each learning rate $\eta$, we run SGD for $20$ independent rounds with the same initialization and dataset, and display the distribution of the absolute remainder across rounds by boxplots. Left: the first-order remainder $|R_\eta(x)|$. Right: the second-order remainder $|R_\eta^{(2)}(x)|$.}
    \label{fig:sgd_val}
\end{figure}

Figure~\ref{fig:sgd_val} reports the empirical scaling of the SGD reminder at the test point $x=1.2$. We consider the learning rates
$\eta\in\left\{1,0.5,0.2,0.1,0.05\right\}$. For each learning rate $\eta$, we perform 20 independent experiments using the same dataset and initialization with independently sampled mini-batch sequences, and visualize the resulting distribution of the remainder by boxplots. We 
define the second-order remainder by
\begin{align*}
    R_\eta^{(2)}:=&R_\eta(x)+\frac{\eta}{2N}\sum_{n=1}^N
    \mathbb{E}\!\Bigg[ \int_0^{T} K_t^{\text{cur}}(x,x_n) \frac{\partial \ell}{\partial f}\bigl(f(x_n,\bZ_s), y_n^*\bigr)  \, dt \Bigg]\\&-\frac{\eta}{2}\mathbb{E}\!\Bigg[ \int_0^{T} \operatorname{Tr}\!\Big(\nabla^2_\bTheta f(x, \mathbf{Z}_t)\;\bs_{|\mathcal{B}|}(\mathbf{Z}_t)\Big)\, dt \Bigg].
\end{align*}
In the left panel of Figure~\ref{fig:sgd_val}, the median of $|R_\eta(x)|$ increases approximately linearly with $\eta$, which is broadly consistent with the predicted first-order scaling. In the right panel, the second-order remainder shows a qualitatively similar trend to the reference slope $2$, although the spread across runs remains visible. Overall, this provides numerical evidence that, under SGD training, the first and second-order remainder terms follow the scaling behavior predicted by Theorem~\ref{2-order Domingo}.

\begin{figure}
    \centering
    \subfigure[]{
    \includegraphics[width=0.48\linewidth]{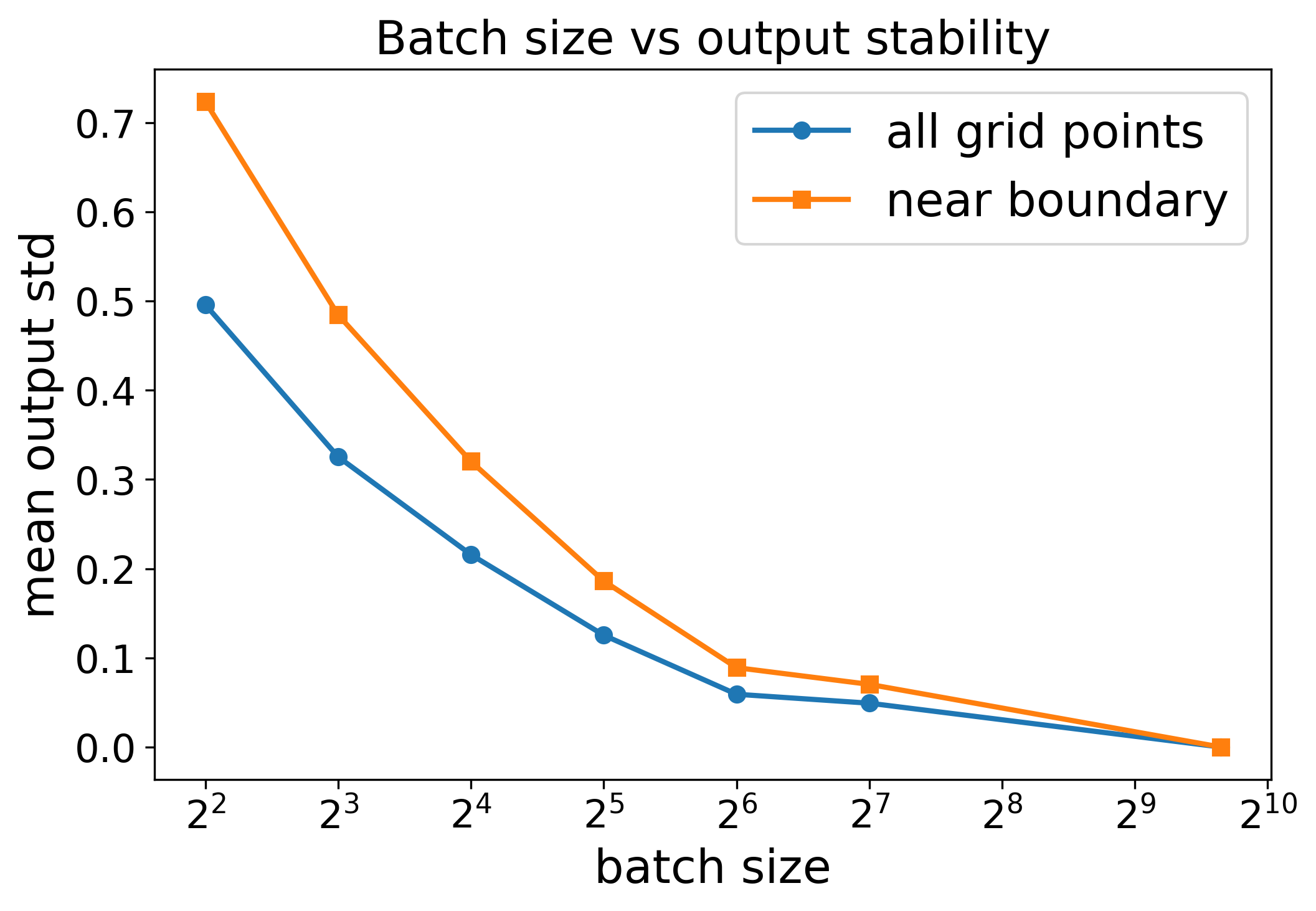}}
    \subfigure[]{
    \includegraphics[width=0.48\linewidth]{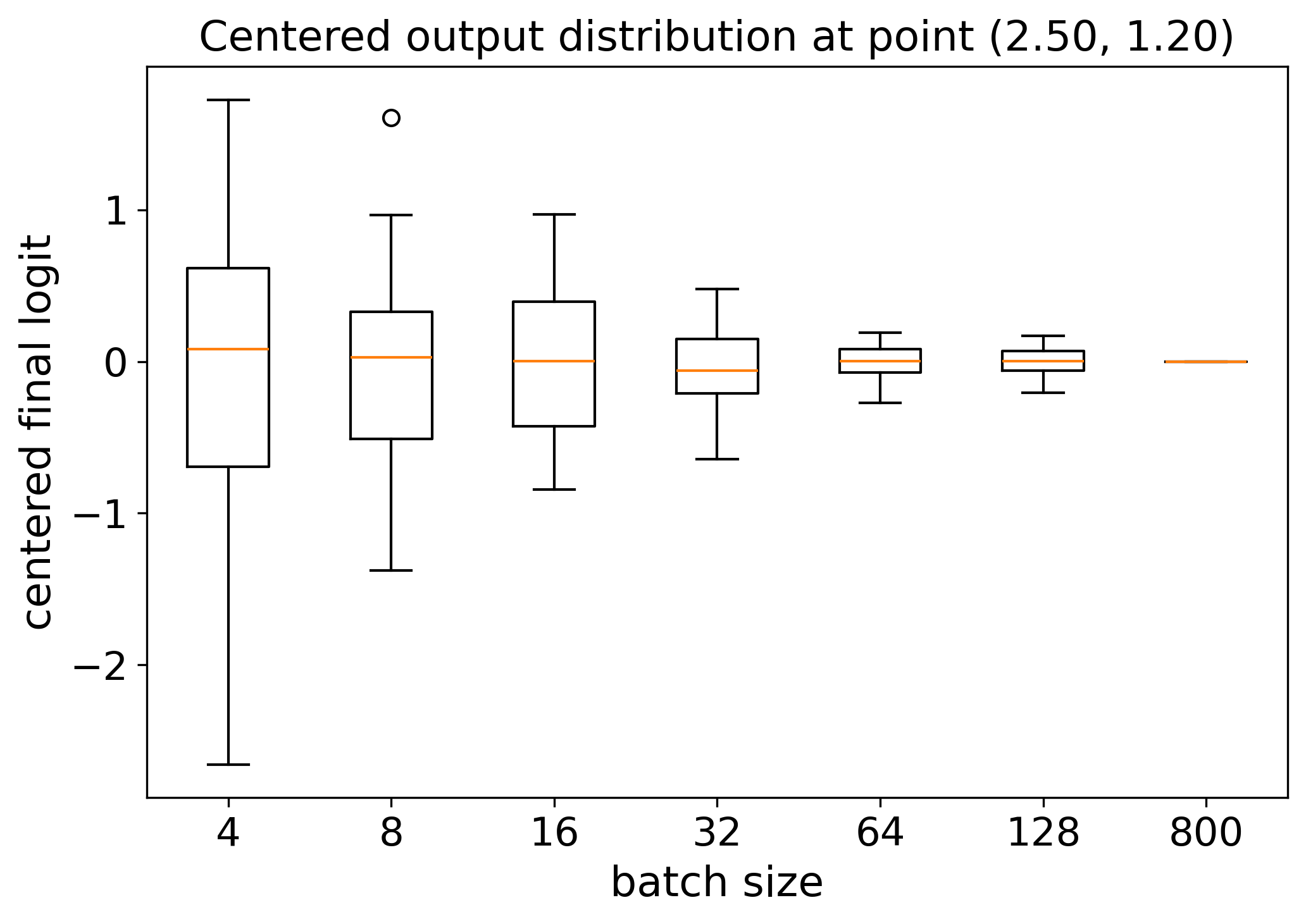}}

    \subfigure[]{
    \includegraphics[width=0.98\linewidth]{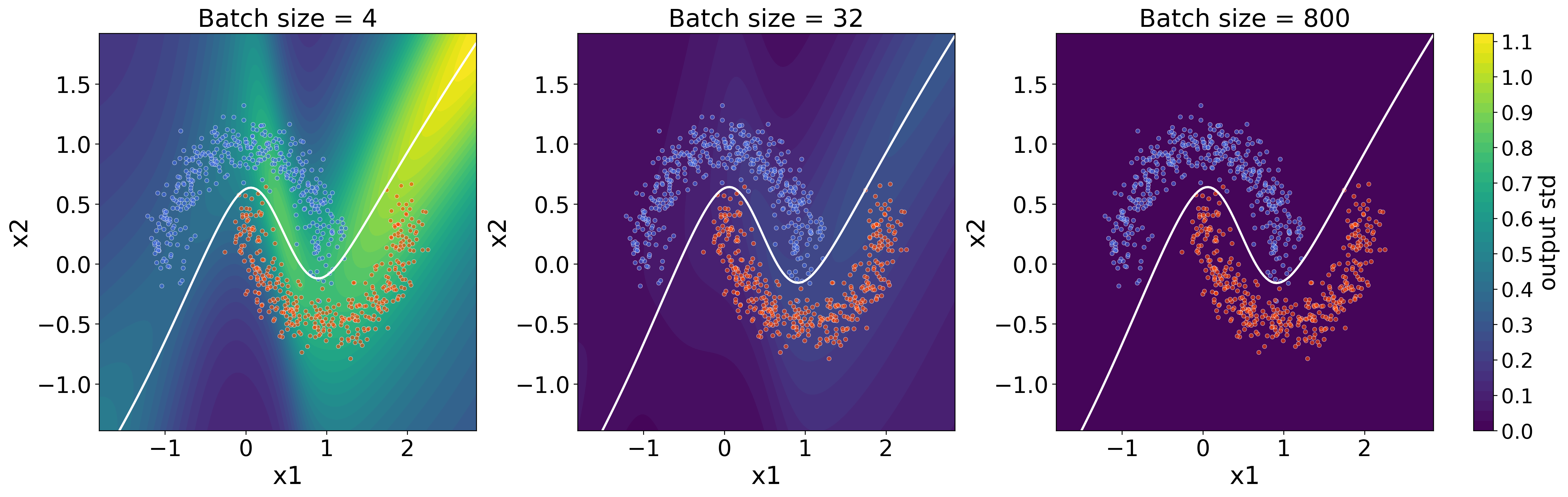}}
    \caption{Effect of batch size on output stability for SGD on the two-moons classification task. A two-hidden-layer tanh network of width \(32\) trained by SGD, while the batch size varies over \(\{4,8,16,32,64,128,800\}\). (a) Mean standard deviation of the final logit across runs decreases monotonically as the batch size increases. (b) Centered final-logit distribution at the representative point \((2.5,1.2)\). (c) Heatmaps of the output standard deviation for batch sizes \(4\), \(32\), and \(800\).}
    \label{fig:batch_stability}
\end{figure}

\subsection{Batch size and  sensitivity of predictions}
We next study how the mini-batch size affects output stability in stochastic training on a nonlinear classification task. We use the two-moons dataset with $1000$ samples and noise level $0.12$, of which $800$ samples are used for training. The classifier is a two-hidden-layer tanh network of width $32$, trained by SGD with learning rate $0.05$. To align the comparison across batch sizes with the discrete-time training index, we fix the total number of parameter updates to \(3000\) and vary the batch size in
\[
|\mathcal B|\in\{4,8,16,32,64,128,800\}.
\]
For each batch size, we repeat $20$ training rounds from the same initialization, changing only the mini-batch sampling order. For a test point \(x\), we define output stability through the variability of the final logit \(f(x,\bTheta_k)\) between runs, where \(k=3000\) is the total number of SGD iterations. Figure~\ref{fig:batch_stability} summarizes the results. Panel (a) shows that the mean output standard deviation decreases monotonically as the batch size increases. This trend is particularly pronounced for test points near the decision boundary, indicating that the stability is sensitive to test data. Panel (b) displays the centered final-logit distribution at the representative test point \((2.5,1.2)\), where the spread between runs is much larger for small batch sizes and collapses as the batch size approaches full-batch training. Panel (c) visualizes the spatial distribution of the output standard deviation in the input domain for batch sizes $4$, $32$, and $800$. The instability is strongest for small batch sizes and is progressively reduced as the batch size grows, becoming nearly negligible in the full-batch regime. These results provide direct numerical evidence that mini-batch sampling noise significantly affects prediction stability and that increasing the batch size stabilizes the network output.

\begin{figure}
    \centering
    \includegraphics[width=0.45\linewidth]{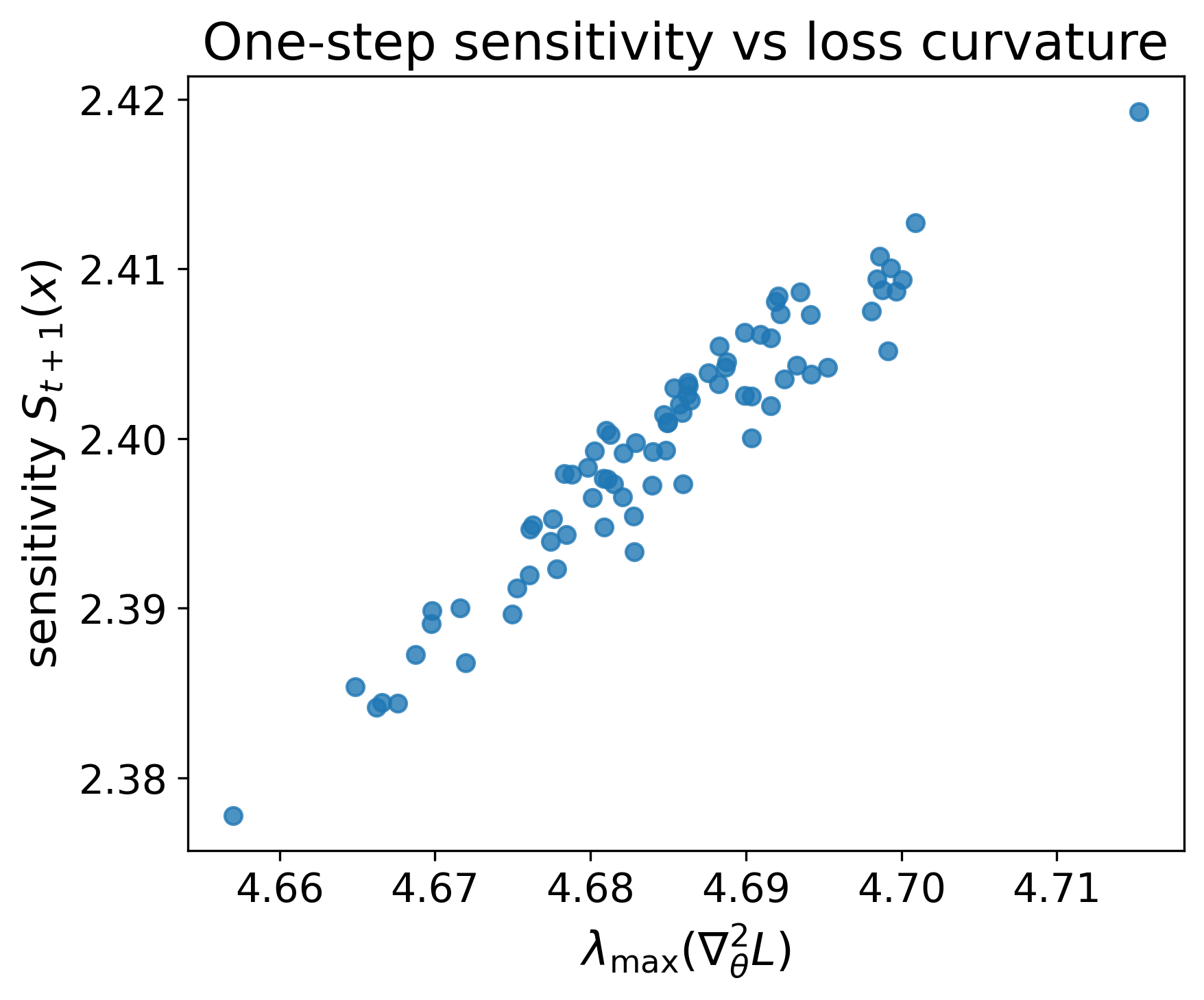}
    \includegraphics[width=0.45\linewidth]{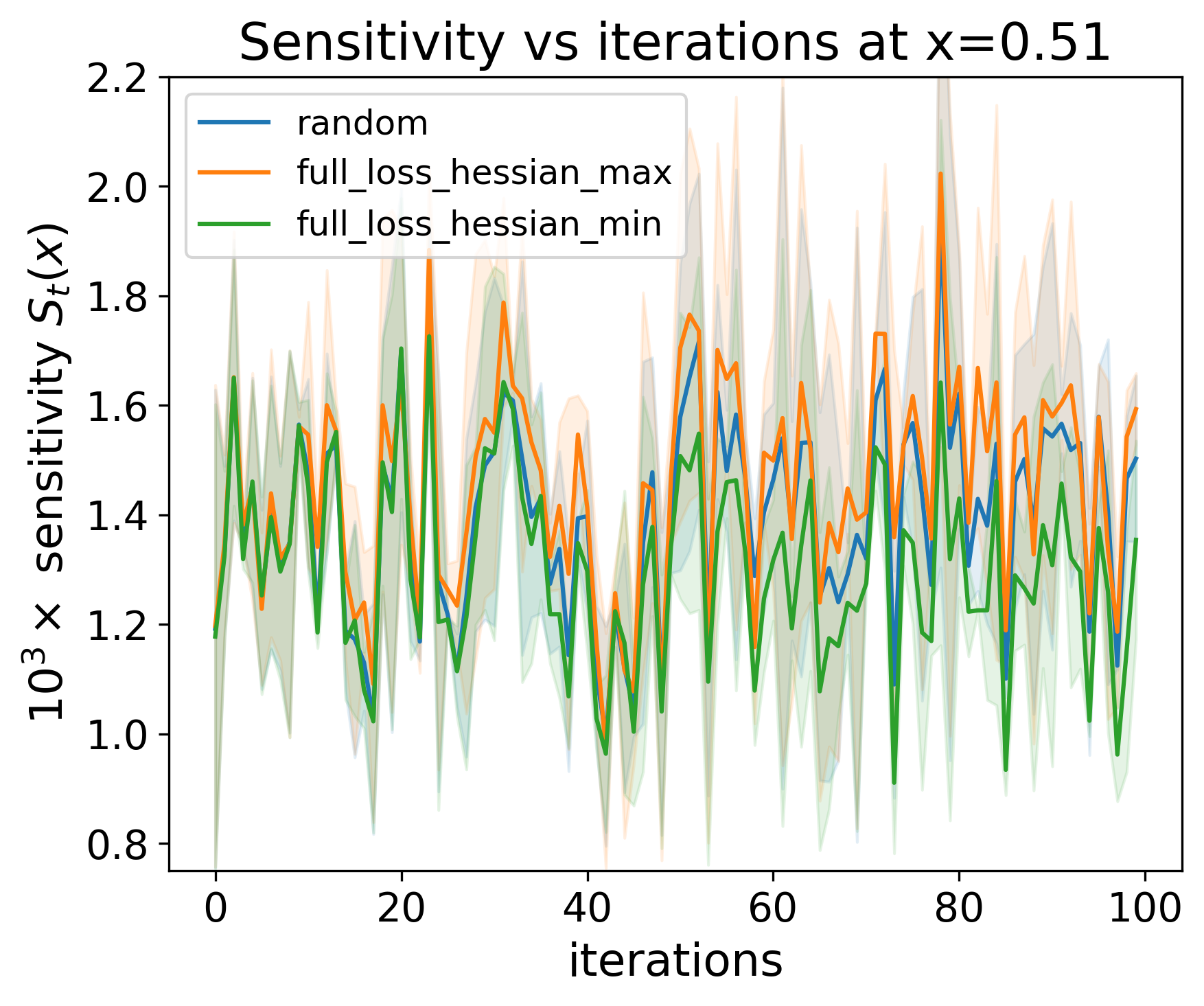}
    \caption{Effect of loss-Hessian-based batch selection on output sensitivity. 
(a) One-step experiment: from a fixed checkpoint, candidate mini-batches are used to perform one SGD update, and the resulting \(S_{t+1}(x)\) is measured. The sensitivity is positively correlated with the largest eigenvalue of the full loss Hessian.
(b) Multi-step experiment: we compare random batch selection with strategies that repeatedly choose the batch with the largest or smallest loss Hessian. The plotted trajectories of \(S_t(x)\) show that batch selection changes the stability of the prediction along training.}
    \label{fig:placeholder}
\end{figure}

\subsection{Sensitivity control via curvature}
We conduct a controlled one dimensional regression experiment to examine how curvature guided mini-batch selection affects prediction stability. Training data are generated from the target function
~\eqref{eq:example_1d} with additive Gaussian  noise of standard deviation \(0.05\). We use \(N=100\) training samples uniformly distributed in the interval \([-2,2]\). The network is~\eqref{eq:network_1d} with width \(M=24\). Training is performed for $100$ iterations with a mini-batch size $40$ and a learning rate \(10^{-2}\). For each strategy, we run the experiment five times. In the sensitivity evaluation, parameter perturbations are sampled from an isotropic Gaussian with standard deviation \(10^{-3}\), and the sensitivity statistics are estimated from \(20\) perturbation samples.

To compare different geometric biases induced by stochastic training, we consider the following batch selection rules. At every iteration, instead of directly drawing one mini batch and updating the network, we first sample \(|\mathcal{B}|=20\) candidate mini batches of size \(40\). For each candidate batch \(\{\mathcal{B}_i\}_{i=1}^{|\mathcal{B}|}\), we compute the one step SGD update
\[
\bTheta^{(\mathcal{B}_i)}_{t+1}=\bTheta_t-\eta \nabla_\bTheta L_{\mathcal{B}_i}(\bTheta_t),
\]
where \(L_{\mathcal{B}_i}\) is the empirical loss on batch \(\mathcal{B}_i\). We then evaluate a curvature score at the one step look ahead parameter \(\bTheta^{(\mathcal{B}_i)}_{t+1}\), and choose the candidate batch that maximizes or minimizes this score. We use two different curvature scores. The first is the largest eigenvalue of the full loss Hessian,
\[
\lambda_{\max}\!\bigl(\nabla_\bTheta^2 L(\bTheta^{(\mathcal{B}_i)}_{t+1})\bigr),
\]
the second is the smallest eigenvalue of the full loss Hessian $\lambda_{\min}\!\bigl(\nabla_\bTheta^2 L(\bTheta^{(\mathcal{B}_i)}_{t+1})\bigr)$.
For a fixed test point \(x\), we define the pointwise sensitivity at iteration \(t\) by perturbing the trained parameter \(\bTheta_t\) with isotropic Gaussian noise \(\xi \sim \mathcal N(0,\sigma_{\mathrm{sens}}^2 I)\), where \(\sigma_{\mathrm{sens}}\) controls the perturbation scale, and measuring the magnitude of the prediction change,
\[
S_t(x):=\mathbb E_{\xi}\bigl|f(x,\bTheta_t+\xi)-f(x,\bTheta_t)\bigr|.
\]

Figure~\ref{fig:placeholder} reports the pointwise sensitivity \(S_t(x)\) defined above. The one-step experiment shows that batches leading to larger loss-Hessian eigenvalues tend to produce larger post-update sensitivity \(S_{t+1}(x)\).
For a fixed test point $x=0.51$, the maximum-eigenvalue strategy typically yields the largest sensitivity trajectory, while minimum-eigenvalue strategy stays lower and closer to the random baseline. This indicates that forcing the optimization path toward larger curvature of the full training objective leads to more fragile predictions under parameter perturbations, whereas selecting batches that reduce the curvature of the full loss leads to more stable predictors. Overall, it supports the claim that flatter minima, as measured through the geometry of the full loss, are associated with improved prediction stability.

We next study whether the curvature-based batch selection affects the sensitivity of the loss to parameter perturbations in a small-scale MNIST experiment. We use 50 training images per class. At each SGD iteration, we sample 10 candidate mini-batches of size 5, perform a one-step look-ahead update for each candidate, estimate the largest eigenvalue of the loss Hessian and select the batch with either maximal or minimal curvature. We repeat each experiment for 5 runs and record the loss sensitivity every 20 steps.

\begin{figure}
    \centering
    \includegraphics[width=0.45\linewidth]{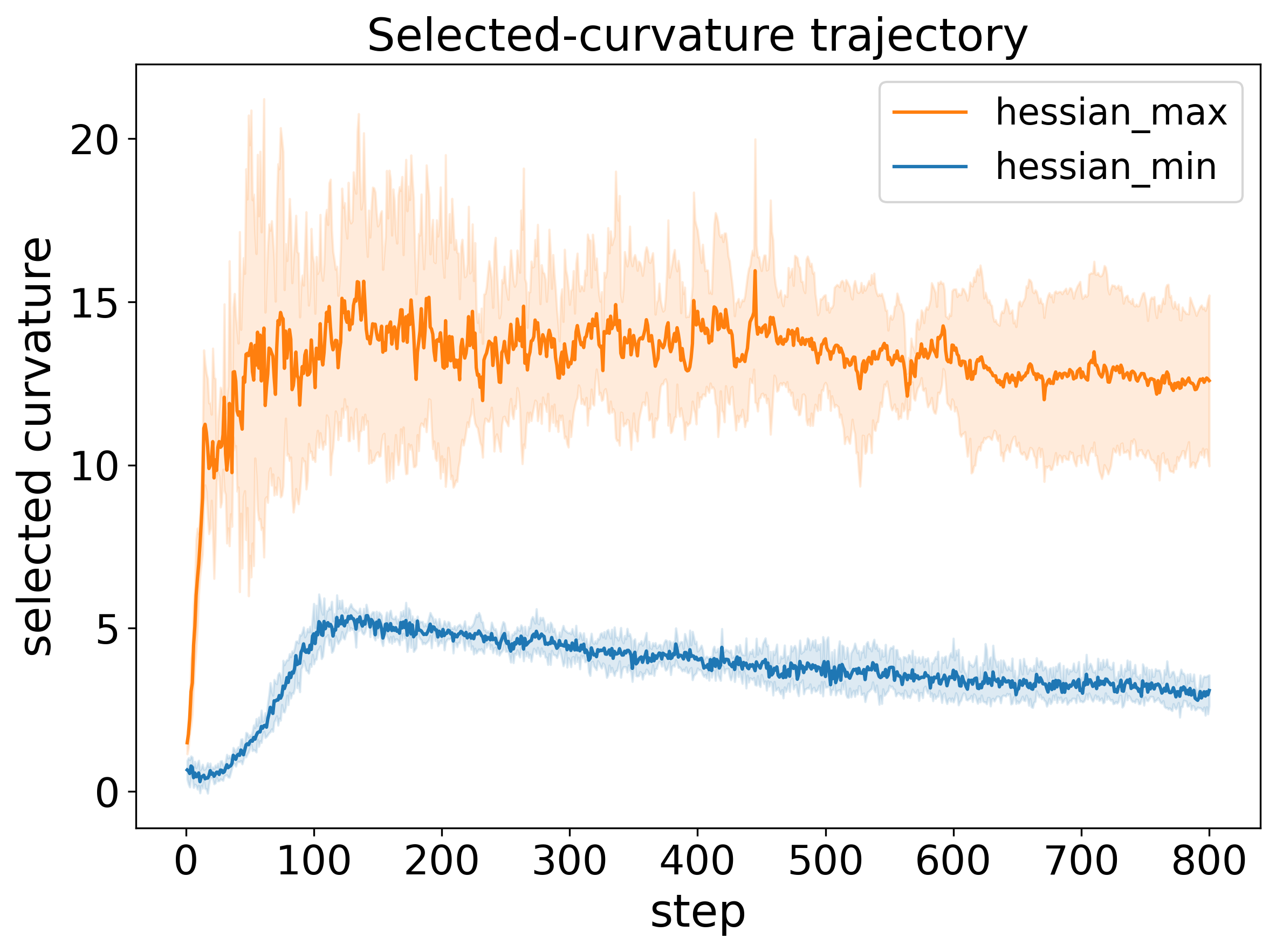}
    \includegraphics[width=0.45\linewidth]{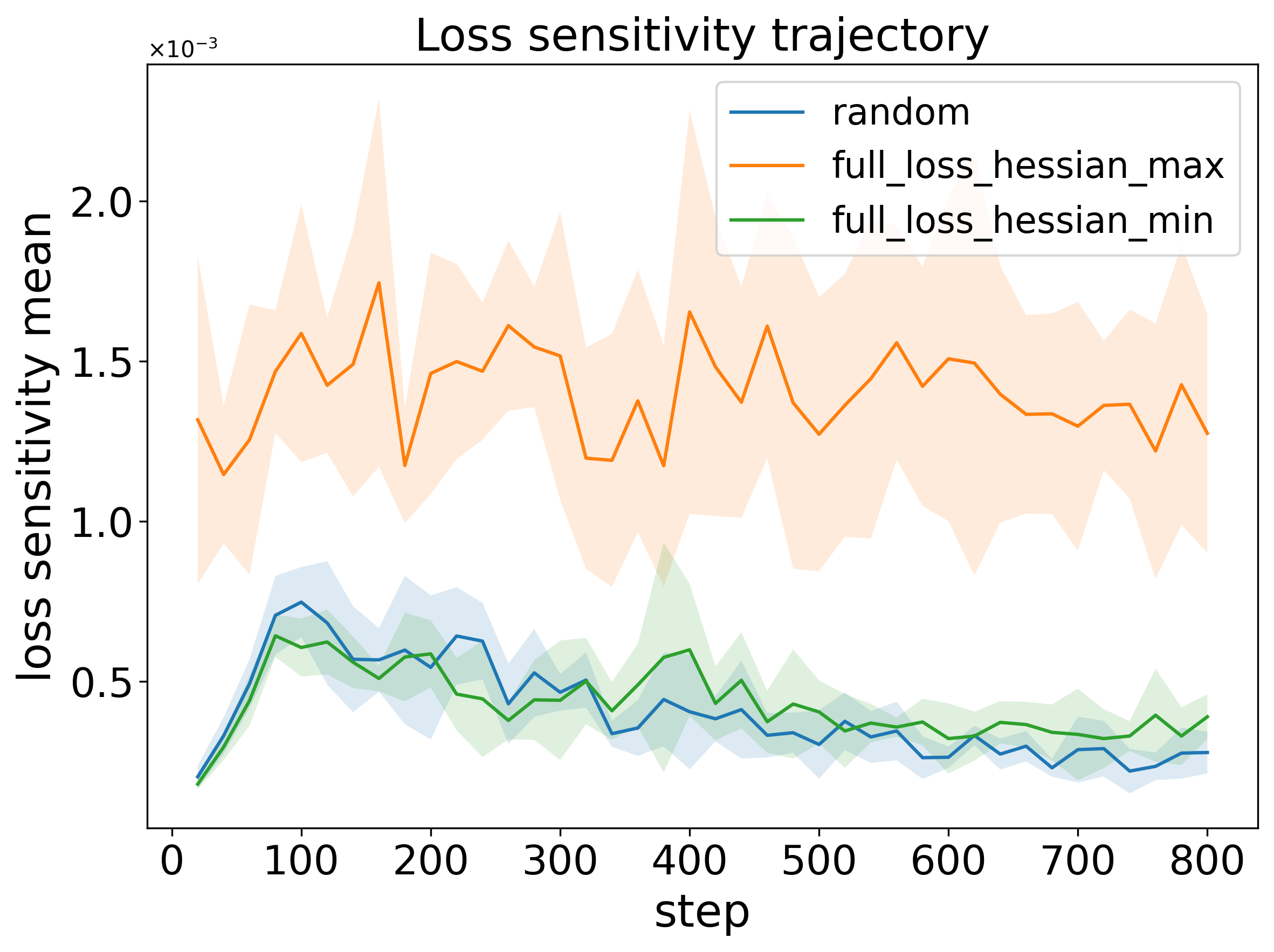}
    \caption{ Curvature-based batch selection and loss sensitivity on a reduced MNIST task. We train on MNIST with 50 training images per class. Left: selected-curvature trajectory. Right: loss-sensitivity trajectory, computed from Gaussian parameter perturbations with $\sigma_{\text{sens}}=10^{-3}$.}
    \label{fig:mnist_cur}
\end{figure}

Figure~\ref{fig:mnist_cur} summarizes the resulting training dynamics. The left panel shows the trajectory of the selected curvature. The maximal-curvature rule consistently chooses batches with substantially larger curvature values than the minimal-curvature rule throughout training. The right panel shows the corresponding trajectory of the loss sensitivity. We observe that the maximum-eigenvalue strategy also leads to a consistently higher loss-sensitivity trajectory than both the minimal-eigenvalue strategy rule and the standard SGD baseline. This suggests that the curvature bias introduced by the batch selection rule has a direct impact on the sensitivity of the loss to parameter perturbations.

\section{Conclusion}\label{conclusion}
In this work, we developed second-order output representations for models trained by gradient-based optimization. Starting from the path-kernel interpretation of gradient descent, we show that the second-order expansion can still be viewed as an interpolation of training residuals along the optimization trajectory. The interpolation weights are refined by an additional curvature-induced component, which records how the loss geometry modifies the propagation of training residuals to a test prediction. Under stochastic training, a further sampling-induced contribution appears, showing that mini-batch sampling affects the expected prediction through the parameter curvature of the prediction map along the directions of the sampling covariance. For SGDM, these mechanisms are propagated further through a momentum-induced memory effect. The concentration result complements the second-order representation by showing how the learning rate, the batch-size-dependent covariance, and the loss Hessian control the size of prediction fluctuations. Our numerical experiments support these conclusions: the observed scaling of the first- and second-order remainders is consistent with the theory. Overall, the network stores information about the training data through the accumulated residuals along the optimization path, and the final prediction is obtained by interpolating these residuals through data-dependent kernel weights. The second-order viewpoint yields a refined interpolation framework, in which a prediction is shaped not only by path kernels but also by curvature- and sampling-induced contributions.

\appendix
\section{Second-order representation with scheduler}\label{sec:second-scheduler}
A widely used learning‑rate schedule in deep learning is the cosine scheduler~\cite{loshchilov2016sgdr}, which smoothly anneals the step size from a large initial value to a small final value according to
\[
\eta_k = \eta_{\min} + \frac12 (\eta_{\max} - \eta_{\min})
\Bigl(1 + \cos\!\,\bigl(\tfrac{k}{K}\pi\bigr)\Bigr),\qquad k = 0,\dots,K,
\]
where \(\eta_{\max}>0\) and \(\eta_{\min}\ge 0\) are the maximum and minimum learning rates, \(T\) is the total number of training iterations, and \(t \in \{0,1,\dots,T\}\) is the iteration index. Unlike a constant step size or a piecewise decay schedule, cosine annealing produces a smooth and progressively decreasing learning rate. Thus, the optimization initially takes relatively large steps, while the update magnitude gradually decreases during the later stages of training, reducing the possibility of oscillatory behavior near a minimizer.

To apply the second-order interpolation formula to this setting, we first identify the modified equation associated with gradient descent under a schedule learning-rate weight.
\begin{lemma}
    Let $L\in\mathcal{C}^3(\mathbb{R}^d)$. For scheduler learning rate $\eta_k = \eta_{\min} + \frac12 (\eta_{\max} - \eta_{\min}) \Bigl(1 + \cos\!\,\bigl(\tfrac{k}{K}\pi\bigr)\Bigr)$, define $\eta=\eta_{\min}$, let \(t_k=k\eta\), \(k=0,\ldots,K\), with \(T=K\eta\). Consider the scheduled gradient-descent iteration
    \begin{equation}\label{eq:discrete_theta}
        \bTheta_{k+1}=\bTheta_k-\eta_k \nabla_\bTheta L(\bTheta_k),\,\, k=0,\cdots,K-1.
    \end{equation}
 Define scalar weight $w(t)= 1 + \tfrac12(\kappa-1)\bigl(1 + \cos(\tfrac{\pi t}{T})\bigr), \kappa = \eta_{\max}/\eta_{\min},\,\,\,t\in[0,T]$. Equivalently, $\eta_k=\eta w(t_k)$.
    Then the corresponding modified equation is
    \begin{equation}\label{eq:modi_theta_continuous}
        \dot{\bTheta}_t=-w(t)\nabla L(\bTheta_t)+\frac{\eta}{2}\dot{w}(t)\nabla L(\bTheta_t)-\frac{\eta}{2}w(t)^2\nabla^2L(\bTheta_t)\nabla L(\bTheta_t)+O(\eta^2).
    \end{equation}
\end{lemma}
\begin{proof}
    Denote $b(t,\bTheta_t)=-w(t)\nabla_\bTheta L(\bTheta_t)$, following the standard backward error analysis, we determine $a(t,\bTheta_t)$ to get the truncated modified equation
    \begin{equation}\label{eq:modi_theta_dot}
        \dot{\bTheta}_t=b(t,\bTheta_t)+\eta a(t,\bTheta_t).
    \end{equation}
    Higher-order corrections in the modified vector field start at order $\eta^2$.
    Our method is to determine $a(t,\bTheta)$ by matching the second-order Taylor expansion of this modified equation with the scheduled gradient-descent update.
    Starting from \(t_k=k\eta\), Taylor expansion over one step gives
\begin{equation}\label{eq:taylor_theta}
\bTheta(t_k+\eta)=\bTheta(t_k)+\eta\dot{\bTheta}(t_k)+\frac{\eta^2}{2}\ddot{\bTheta}(t_k)+O(\eta^3).
\end{equation}
    Differentiating Eq.~\eqref{eq:modi_theta_dot} with respect to \(t\), we obtain
    \begin{align}\label{eq:second_derivative1}
        \ddot{\bTheta}_t=&\frac{d}{dt}\left[b(t,\bTheta_t)+\eta\,a(t,\bTheta_t)\right]\notag\\
        =&\partial_t b(t,\bTheta_t)+D_\bTheta b(t,\bTheta_t)\dot{\bTheta}_t+\eta \left[\partial_t a(t,\bTheta_t)+D_\bTheta (a(t,\bTheta_t))\dot{\bTheta}_t\right].
    \end{align}
Since $\dot{\bTheta}_t=b(t,\bTheta_t)+O(\eta),D_\bTheta b(t,\bTheta_t)=-w(t)\nabla^2_\bTheta L(\bTheta_t)$, we have
\begin{align}\label{eq:Db_theta}
      D_\bTheta b(t,\bTheta_t)\dot{\bTheta}_t=&D_\bTheta b(t,\bTheta_t){b}(t,\bTheta_t)+O(\eta)
      =-w(t)\nabla_\bTheta^2 L(\bTheta_t)b(t,\bTheta_t)+O(\eta)\notag\\
      =&\,w(t)^2\nabla_\bTheta^2 L(\bTheta_t)\nabla_\bTheta L(\bTheta_t)+O(\eta).
\end{align}
  Plugging Eq.(\ref{eq:Db_theta}) into Eq.(\ref{eq:second_derivative1}), we have
  \begin{align}\label{eq:second_derivative2}
      \ddot{\bTheta}_t=\partial_t b(t,\bTheta_t)+w(t)^2\nabla_\bTheta^2 L(\bTheta_t)\nabla_\bTheta L(\bTheta_t)+O(\eta).
  \end{align}
  Then Eq.(\ref{eq:taylor_theta}) becomes
  \begin{align}
      \bTheta(t_k+\eta)=&\bTheta(t_k)+\eta b(t_k,\bTheta_{t_k})+\eta^2\bigg[a(t_k,\bTheta_{t_k})+\frac{1}{2}\partial_t b(t_k,\bTheta_{t_k})\\&+\frac{1}{2}(D_\bTheta b(t_k,\bTheta_{t_k}))b(t_k,\bTheta_{t_k})\bigg]\notag+O(\eta^3).
  \end{align}
  Then we choose $a(t,\bTheta_t)$ so that the coefficient of $\eta^2$ vanishes, that is
  \begin{equation*}
      a(t,\bTheta_t)+\frac{1}{2}\partial_t b(t,\bTheta_t)+\frac{1}{2}(D_\bTheta b(t,\bTheta_t))b(t,\bTheta_t)=0.
  \end{equation*}
  Substituting the expressions for $\partial_t b(t,\bTheta_t)$ and  $D_\bTheta b(t,\bTheta_t)$, we obtain
  \begin{equation}\label{eq:a_solve}
      a(t,\bTheta_t)=\frac12\dot w(t)\nabla_\bTheta L(\bTheta_t)-\frac12w(t)^2\nabla^2_\bTheta L(\bTheta_t)\nabla_\bTheta L(\bTheta_t).
  \end{equation}
  Plugging Eq.(\ref{eq:a_solve}) into Eq.(\ref{eq:modi_theta_dot}), the truncated modified equation is
  \begin{equation*}
      \dot{\bTheta}_t=-w(t)\nabla L(\bTheta_t)+\frac{\eta}{2}\dot{w}(t)\nabla L(\bTheta_t)-\frac{\eta}{2}w(t)^2\nabla^2L(\bTheta_t)\nabla L(\bTheta_t).
  \end{equation*}
Since $T$ is fixed, the $O(\eta^3)$ local error accumulated over
$K=T/\eta$ steps gives an $O(\eta^2)$ error at the grid points
$t_k=k\eta$, under the usual local Lipschitz bounds.
\end{proof}

Applying Theorem~\ref{thm_domingo} to this time-rescaled gradient flow (Eq.(\ref{eq:modi_theta_continuous})) yields the following scheduler-weighted path-kernel evolution.

\begin{corollary}[Scheduler-weighted path-kernel evolution]\label{cor:scheduler}
Assume the same regularity conditions as in Theorem~\ref{thm_domingo}, and consider the modified equation of the cosine learning-rate schedule with learning rate $w(t)$. Then for any model output $f(x,\bTheta_T)$, the terminal prediction satisfies
\begin{align}
    f(x,\bTheta_T)
=&f(x,\bTheta_0)
-\frac{1}{N}\sum_{n=1}^{N}\,\int_{0}^{T} 
\left(w(t)-\frac{\eta}{2}\dot{w}(t)\right)
K_t(x,x_n)
\frac{\partial \ell}{\partial f}\bigl(f(x_n,\bTheta_t),y_n^*\bigr) \,dt\\&-\frac{\eta}{2N}\sum_{n=1}^N\int_{0}^{T} w(t)^2 K^{\text{cur}}_t(x,x_n)\frac{\partial \ell}{\partial f}\bigl(f(x_n,\bTheta_t),y_n^*\bigr)dt+O(\eta^2),
\end{align}
with the scalar weight $w(t)= 1 + \tfrac12(\kappa-1)\bigl(1 + \cos(\tfrac{\pi t}{T})\bigr), t\in[0,T].$
\end{corollary}
This expression shows that cosine annealing does not alter the path‑kernel structure, while introducing time-dependent weights into both the leading and curvature-dependent contributions. The leading path-kernel term is weighted by $w(t)-\frac{\eta}{2}\dot{w}(t)$, the second-order curvature-weighted term is scaled by $w(t)^2$.

\section{Lemmas}
We collect some technical lemmas used for proving our main results. 
\begin{lemma}\label{lem:backward-gradient-bound}
Let \(T>0\). For \(0\le t\le T\) and \(z\in\mathcal U\), let
\(\bZ_s^{t,z}\), \(s\in[t,T]\), be the solution of
\[
d\bZ_s^{t,z}=b_\eta(\bZ_s^{t,z})\,ds+\sqrt{\eta}\,\bs_{|\mathcal{B}|}^{1/2}(\bZ_s^{t,z})\,d\bW_s,
\qquad
\bZ_t^{t,z}=z,
\]
where $b_\eta(z)=-\nabla L(z)-\frac{\eta}{2}\nabla^2L(z)\nabla L(z)$.
For a fixed test point \(x\), define $u_x(t,z)=\mathbb E\left[f(x,\bZ_T^{t,z})\right]$. Assume that the restarted trajectory \(\bZ_s^{t,z}\) remains in \(\mathcal U\), and that the following bounds hold on \(\mathcal U\):
\[
\|\nabla L(z)\|\le G_L,\qquad
\|\nabla^2L(z)\|_{\mathrm{op}}\le M_2,\qquad
\|\nabla^3L(z)\|_{\mathrm{op}}\le M_3,
\]
and $\|\nabla_\bTheta f(x,z)\|\le G_f(x)$ for the fixed test point $x$.
Assume further that the diffusion coefficient \(\bs_{|\mathcal{B}|}^{1/2}\) satisfies
\[
\sum_{q=1}^d
\left\|
D(\bs^{1/2}_{|\mathcal{B}|})_q(z)v
\right\|^2
\le
K_\Sigma^2\|v\|^2,
\qquad
\forall z\in\mathcal U,\ \forall v\in\mathbb R^d,
\]
where \((\bs^{1/2}_{|\mathcal{B}|})_q(z)\) denotes the \(q\)-th column of
\(\bs^{1/2}_{|\mathcal{B}|}(z)\).
Set $\widetilde K_\eta=M_2+\frac{\eta}{2}(M_3G_L+M_2^2)+\frac{\eta}{2}K_\Sigma^2$.
Then
\[
\|\nabla_z u_x(t,z)\|
\le
G_f(x) e^{\widetilde K_\eta(T-t)}.
\]
\end{lemma}

\begin{proof}
First, we estimate the derivative of the drift. Since
\[
b_\eta(z)=-\nabla L(z)-\frac{\eta}{2}\nabla^2L(z)\nabla L(z),
\]
we have
\[
\nabla b_\eta(z)=-\nabla^2L(z)-\frac{\eta}{2}
\nabla\left(\nabla^2L(z)\nabla L(z)\right).
\]
For any direction \(v\), 
\[
\nabla\left(\nabla^2L(z)\nabla L(z)\right)v
=(\nabla_v\nabla^2L(z))\nabla L(z)+\nabla^2L(z)\nabla^2L(z)v.
\]
Therefore,
\[
\left\|
\nabla\left(\nabla^2L(z)\nabla L(z)\right)
\right\|_{\mathrm{op}}
\le
\|\nabla^3L(z)\|_{\mathrm{op}}\|\nabla L(z)\|+\|\nabla^2L(z)\|_{\mathrm{op}}^2.
\]
Using the bounds on \(\nabla L\), \(\nabla^2L\), and \(\nabla^3L\), we obtain
\[\|\nabla b_\eta(z)\|_{\mathrm{op}}\le M_2+\frac{\eta}{2}(M_3G_L+M_2^2)=:K_b.\]

Let $J_{s,t}(z)=\nabla_z \bZ_s^{t,z}$
be the derivative of the stochastic flow with respect to the initial point.
Differentiating the SDE with respect to \(z\), we obtain
\[
dJ_{s,t}(z)
=
\nabla b_\eta(\bZ_s^{t,z})J_{s,t}(z)\,ds+\sqrt{\eta}
\sum_{q=1}^dD(\bs_{|\mathcal{B}|}^{1/2})_q(\bZ_s^{t,z})J_{s,t}(z)\,dW_s^q,
\qquad
J_{t,t}(z)=\mathbf{I}.
\]

Fix a unit vector \(v\in\mathbb R^d\), and set $Y_s=J_{s,t}(z)v$.
Then
\[
dY_s=\nabla b_\eta(\bZ_s^{t,z})Y_s\,ds+\sqrt{\eta}
\sum_{q=1}^d
D(\bs_{|\mathcal{B}|}^{1/2})_q(\bZ_s^{t,z})Y_s\,dW_s^q.
\]
Applying Itô's formula to \(\|Y_s\|^2\) gives
\[
\begin{aligned}
d\|Y_s\|^2
=&2Y_s^\top\nabla b_\eta(\bZ_s^{t,z})Y_s\,ds+
\eta
\sum_{q=1}^d
\left\|
D(\bs_{|\mathcal{B}|}^{1/2})_q(\bZ_s^{t,z})Y_s
\right\|^2 ds\\&+
2\sqrt{\eta}
\sum_{q=1}^d
Y_s^\top
D(\bs_{|\mathcal{B}|}^{1/2})_q(\bZ_s^{t,z})Y_s\,dW_s^q.
\end{aligned}
\]
Taking expectations, the stochastic integral has zero expectation. Hence
\[
\frac{d}{ds}\mathbb E\|Y_s\|^2
=
\mathbb E\left[
2Y_s^\top\nabla b_\eta(\bZ_s^{t,z})Y_s
+
\eta
\sum_{q=1}^d
\left\|
D(\bs_{|\mathcal{B}|}^{1/2})_q(\bZ_s^{t,z})Y_s
\right\|^2
\right].
\]
Using $\|\nabla b_\eta(z)\|_{\mathrm{op}}\le K_b$
 and $\sum_{q=1}^d
\left\|D(\bs_{|\mathcal{B}|}^{1/2})_q(z)v
\right\|^2\le K_\Sigma^2\|v\|^2,$
we get
\[
\frac{d}{ds}\mathbb E\|Y_s\|^2
\le (2K_b+\eta K_\Sigma^2)\mathbb E\|Y_s\|^2.
\]
Since \(Y_t=v\) and \(\|v\|=1\), Gronwall's inequality yields
\[
\mathbb E\|J_{s,t}(z)v\|^2
\le e^{(2K_b+\eta K_\Sigma^2)(s-t)}=e^{2\widetilde K_\eta(s-t)},
\]
where $\widetilde K_\eta=K_b+\frac{\eta}{2}K_\Sigma^2$. Next, by differentiating $u_x(t,z)=\mathbb E[f(x,\bZ_T^{t,z})]$
with respect to \(z\), we obtain
\[
\nabla_z u_x(t,z)
=\mathbb E\left[
J_{T,t}(z)^\top\nabla_\bTheta f(x,\bZ_T^{t,z})
\right].
\]
Thus, for every unit vector \(v\),
\[
\begin{aligned}
|v^\top\nabla_z u_x(t,z)|
&=
\left|
\mathbb E\left[\nabla_\bTheta f(x,\bZ_T^{t,z})^\top J_{T,t}(z)v\right]
\right|\\
&\le G_f(x)\mathbb E\left[\|J_{T,t}(z)v\|\right]\\
&\le G_f(x)\left(\mathbb E\|J_{T,t}(z)v\|^2\right)^{1/2}\\
&\le G_f(x) e^{\widetilde K_\eta(T-t)}.
\end{aligned}
\]
Taking the supremum over all unit vectors \(v\), we obtain
\[
\|\nabla_z u_x(t,z)\| \le G_f(x) e^{\widetilde K_\eta(T-t)}.
\]
This completes the proof.
\end{proof}

We next eastimate the stochastic integral that appears in the SGDM expansion.
\begin{lemma}\label{lem:stochastic-convolution-bound}
    Assume that there exists $C_T>0$, independent of $\eta$, such that for every \(z\in\{\bZ_r:0\le r\le T\}\),
    \begin{align*}
        &\|\bs^{1/2}_{|\mathcal{B}|}(z)\|_F\le C_T, \qquad\|D(\bs^{1/2}_{|\mathcal{B}|}(z))[v]\|_{\mathrm{F}}\le C_T\|v\|,\forall v\in\mathbb{R}^d,
\\
&\|\nabla_\bTheta^2L(z)\|_{\mathrm{op}}+\|\nabla_\bTheta^3L(z)\|_{\mathrm{op}}
\le C_T,\quad\|\nabla_\bTheta^2 f(x,z)\|_{\mathrm{op}}\le C_T,\quad \sup_{0\le r\le T}\|\bM_r\|\le C_T.
    \end{align*}
Here and below, \(C_T\) denotes a generic constant, independent of \(\eta\), whose value may change from line to line.
    Then for $0\le t\le T$,
    \begin{equation*}
        \left|
\mathbb E\left[
\nabla_\bTheta f(x,\bZ_t)^\top
\int_0^t \bs^{1/2}_{|\mathcal{B}|}(\bZ_s)\,d\bW_s
\right]
\right|
\le C_T\sqrt{\eta}.
    \end{equation*}
\end{lemma}
\begin{proof}
    Since the Brownian noise enters only through the momentum $\bM$, we have
    $$D_s \bM_s=\sqrt{\eta}\bs_{|\mathcal{B}|}^{\frac{1}{2}}(\bZ_s),\, D_s \bZ_s=0.$$
    Taking the Malliavin derivative of $M_r$ gives
    \begin{align*}
\mathrm{D}_s\bM_r=&\sqrt{\eta}\,\bs_{|\mathcal{B}|}^{1/2}(\bZ_s)
-\int_s^r
\left(
\mu I+\frac{\eta}{2}
(\mu^2I-\nabla_\bTheta^2L(\bZ_\zeta))
\right)D_s\bM_\zeta\,d\zeta\\&+\frac{\eta}{2}\int_s^r
\left[
\nabla_\bTheta^3L(\bZ_\zeta)[D_s\bZ_\zeta]
\right]\bM_\zeta\,d\zeta
-\left(1+\frac{\eta}{2}\mu\right)
\int_s^r
\nabla_\bTheta^2L(\bZ_\zeta)D_s\bZ_\zeta\,d\zeta
\\&+\sqrt{\eta}\int_s^r
D\bs_{|\mathcal{B}|}^{1/2}(\bZ_\zeta)[D_s\bZ_\zeta]\,d\bW_\zeta.
    \end{align*}

By It\'o isometry we obtain
\begin{equation*}
    \begin{aligned}
\mathbb E
\left\|
\sqrt{\eta}\int_s^r
D\bs_{|\mathcal{B}|}^{1/2}(\bZ_\zeta)[D_s\bZ_\zeta]\,d\bW_\zeta
\right\|^2
&=\eta\,\mathbb E\int_s^r \left\|
D\bs_{|\mathcal{B}|}^{1/2}(\bZ_\zeta)[D_s\bZ_\zeta]
\right\|_{\mathrm F}^2d\zeta \\
&\le
C_T\eta
\int_s^r \mathbb E\|D_s\bZ_\zeta\|^2d\zeta \le C_T \int_s^r
\mathbb E\|D_s\bZ_\zeta\|^2d\zeta .
\end{aligned}
\end{equation*}

By Cauchy–Schwarz inequality, we have
\begin{align*}
    \mathbb{E}[\|\mathrm{D}_s\bM_r\|^2]\leq C_T{\eta}+C_T\int_s^r\mathbb{E}\left[\|D_s\bM_\zeta\|^2+\|D_s\bZ_\zeta\|^2\right]d\zeta.
\end{align*}

Taking the Malliavin derivative of $\bZ$ gives
\begin{equation*}
D_s\bZ_r=\int_s^r
\left[
\left(1-\frac{\eta}{2}\mu\right)D_s\bM_\zeta -\frac{\eta}{2}\nabla_\bTheta^2L(\bZ_\zeta)D_s\bZ_\zeta
\right]d\zeta .
\end{equation*}
By Cauchy-Schwarz inequality,
\begin{align*}
    \mathbb E\|D_s\bZ_r\|^2
\le C_T\int_s^r
\left(\mathbb E\|D_s\bM_\zeta\|^2+\mathbb E\|D_s\bZ_\zeta\|^2
\right)d\zeta .
\end{align*}
Define $Y_r=\mathbb{E}[\|D_s\bM_r\|^2]+\mathbb{E}[\|D_s\bZ_r\|^2]$ we obtain
\begin{equation*}
    Y_r\le C_T{\eta}+C_T\int_s^r Y_\zeta \mathrm{d}\zeta.
\end{equation*}

By Gronwall’s inequality,
\begin{align}\label{eq:DZ_bound}
    \mathbb{E}[\|D_s\bZ_r\|^2]\leq Y_r\le C_T{\eta}\exp\bigl(C_T(r-s)\bigr)\le C_T{\eta}\exp\bigl(C_T T\bigr)\le C_T {\eta}.
\end{align}

By Malliavin duality,
\begin{align*}
    \mathbb E\left[\nabla_\bTheta f(x,\bZ_t)^\top\int_0^t 
\bs_{|\mathcal{B}|}^{1/2}(\bZ_s)d\bW_s
\right]&=\mathbb E\left[\int_0^t\operatorname{Tr}
\left(D_s(\nabla_\bTheta f(x,\bZ_t))^\top
\bs_{|\mathcal{B}|}^{1/2}(\bZ_s)\right)ds\right]\\
&=\mathbb E\left[\int_0^t\operatorname{Tr}
\left((\nabla_\bTheta^2 f(x,\bZ_t)D_s \bZ_t)^\top
\bs_{|\mathcal{B}|}^{1/2}(\bZ_s)
\right)ds\right].
\end{align*}
Therefore, by Eq.(\ref{eq:DZ_bound}) we obtain
\begin{align*}
   \Biggl|\mathbb E\left[\nabla_\bTheta f(x,\bZ_t)^\top\int_0^t 
\bs_{|\mathcal{B}|}^{1/2}(\bZ_s)d\bW_s
\right] \Biggr|&\leq \mathbb E\left[\int_0^t\left|\operatorname{Tr}
\left((\nabla_\bTheta^2 f(x,\bZ_t)D_s\bZ_t)^\top
\bs^{1/2}_{|\mathcal{B}|}(\bZ_s)
\right)\right|ds\right]\\
&\leq C_T\int_0^t \mathbb E \|D_s\bZ_t\|ds\le C_T\int_0^t (\mathbb E \|D_s\bZ_t\|^2)^{\frac{1}{2}}ds\\
&\leq C_T\sqrt{\eta}.
\end{align*}
\end{proof}

\begin{lemma}\label{lemma:phi}
For $0\le s\le t \le T:=K\eta$, let $\Phi_{t,s}$ be the fundamental matrix associated with the linear system
\begin{equation*}
    \partial_t \Phi_{t,s}=-A_t\Phi_{t,s},\qquad \Phi_{s,s}=\mathbf{I},
\end{equation*}
where $A_t=\mu \mathbf{I}+\frac{\eta}{2}[\mu^2 \mathbf{I}-\nabla_\bTheta^2 L(\bZ_t)]$. Assume that $\sup_{0\le t\le T}\|\nabla_\bTheta^2 L(\bZ_t)\|\le C_T$, where \(C_T\) is independent of \(\eta\). Then
\begin{align}    
    \Phi_{t,s}=e^{-\mu(t-s)}\left(\mathbf{I}+\frac{\eta}{2}\int_s^t \nabla_\bTheta^2L(\bZ_r)dr-\frac{\mu^2}{2}(t-s)\eta\,\mathbf{I}\right)+O(\eta^2),
\end{align}
where the remainder is understood in operator norm.
\end{lemma}
\begin{proof}
We now derive an explicit approximation of $\Phi_{t,s}$. Defining \begin{equation}\Psi_{t,s}:=e^{(\mu+\frac{\eta}{2}\mu^2)(t-s)}\,\Phi_{t,s},\,\, \text{i.e.}\,\,\,\Phi_{t,s}=e^{-(\mu+\frac{\eta}{2}\mu^2)(t-s)}\Psi_{t,s},\label{defphits}\end{equation}
and differentiating with respect to $t$, we have
\begin{align}\label{eq:psi1}
    \partial_t \Phi_{t,s}&=-\left(\mu+\frac{\eta}{2}\mu^2\right)e^{-(\mu+\frac{\eta}{2}\mu^2)(t-s)}\Psi_{t,s}+e^{-(\mu+\frac{\eta}{2}\mu^2)(t-s)}\partial_t \Psi_{t,s}.
\end{align}
On the other hand, plugging Eq.(\ref{defphits}) into Eq.(\ref{eq:odePhi}) we have
\begin{align}\label{eq:psi2}
    \partial_t \Phi_{t,s}=-\left(\mu+\frac{\eta}{2}\mu^2\right)e^{-(\mu+\frac{\eta}{2}\mu^2)(t-s)}\Psi_{t,s}+\frac{\eta}{2}\nabla_\bTheta^2L(\bZ_t)e^{-(\mu+\frac{\eta}{2}\mu^2)(t-s)}\Psi_{t,s}.
\end{align}
Comparing Eq.(\ref{eq:psi1}) and Eq.(\ref{eq:psi2}) we deduce that
\begin{equation}
    \partial_t \Psi_{t,s}=\frac{\eta}{2}\nabla_\bTheta^2L(\bZ_t)\Psi_{t,s},\,\Psi_{s,s}=\mathbf{I}.
    \label{diffpsi}
\end{equation}
Integrating Eq.(\ref{diffpsi}) between $s$ and $t$, we have
\begin{equation}\label{integral_psi}
    \Psi_{t,s}-\mathbf{I}=\frac{\eta}{2}\int_s^t \nabla_\bTheta^2L(\bZ_r)\Psi_{r,s}dr.
\end{equation}
Define 
\begin{equation}\label{defRts}R_{t,s}:=\Psi_{t,s}-\mathbf{I}-\frac{\eta}{2}\int_s^t \nabla_\bTheta^2L(\bZ_r)dr,
\end{equation}
combining with this definition~(\ref{integral_psi}) we obtain 
\begin{equation} \label{eq:residual}
    R_{t,s}=\frac{\eta}{2}\int_s^t \nabla_\bTheta^2L(\bZ_r)(\Psi_{r,s}-\mathbf{I})dr.
\end{equation}
 Using  Eq.(\ref{diffpsi}) and Grönwall inequality, we have
\begin{align*}
    \|\Psi_{r,s}\|&\leq 1+\frac{\eta}{2}\int_s^r \|\nabla_\bTheta^2L(\bZ_q)\|\|\Psi_{q,s}\|dq\leq 1+\frac{\eta C_T}{2}\int_s^r \|\Psi_{q,s}\|dq\\&\leq \exp{\left(\frac{\eta C_T}{2}(r-s)\right)}\leq \exp{\left(\frac{\eta C_T}{2}K\eta\right)}.
\end{align*}
Noticing that $\Psi_{s,s}=\mathbf{I}$, we have
\begin{equation}\label{eq:psi_bound}
    \|\Psi_{r,s}-\mathbf{I}\|\leq \frac{\eta C_T}{2}\int_s^r \exp{\left(\frac{\eta C_T}{2}K\eta\right)} dq=\frac{\eta C_T}{2}\exp{\left(\frac{\eta C_T}{2}K\eta\right)}(r-s).
\end{equation}
Plugging Eq.(\ref{eq:psi_bound}) into Eq.(\ref{eq:residual}), we obtain
\begin{align*}
    \|R_{t,s}\|\le \frac{\eta\, C_T}{2}\int_s^t \frac{\eta C_T}{2}\exp{\left(\frac{\eta C_T}{2}K\eta\right)}(r-s) dr&\leq \frac{\eta^2 C_T^2}{4}\exp{\left(\frac{\eta C_T}{2}K\eta\right)}\int_s^{t}(r-s) dr\\&=O(\eta^2).
\end{align*}
Going back to the definition \eqref{defRts} of $R_{t,s}$, this yields
\begin{align*}
    \Psi_{t,s}=\mathbf{I}+\frac{\eta}{2}\int_s^t \nabla_\bTheta^2L(\bZ_r)dr+O(\eta^2).
\end{align*}
Combining this result with~\eqref{defphits}, we obtain
\begin{align*}
    \Phi_{t,s}&=e^{-(\mu+\frac{\eta}{2}\mu^2)(t-s)}\left(\mathbf{I}+\frac{\eta}{2}\int_s^t \nabla_\bTheta^2L(\bZ_r)dr+O(\eta^2)\right)\\
    &=e^{-\mu(t-s)}e^{-\frac{\eta}{2}\mu^2(t-s)}\left(\mathbf{I}+\frac{\eta}{2}\int_s^t \nabla_\bTheta^2L(\bZ_r)dr+O(\eta^2)\right).
\end{align*}
Taking Taylor expansion to $e^{-\frac{\eta}{2}\mu^2(t-s)}$, we get
\begin{align*}    
    \Phi_{t,s}&=e^{-\mu(t-s)}\left(1-\frac{\mu^2}{2}(t-s)\eta+O(\eta^2)\right)\left(\mathbf{I}+\frac{\eta}{2}\int_s^t \nabla_\bTheta^2L(\bZ_r)dr+O(\eta^2)\right)\notag\\
    &=e^{-\mu(t-s)}\left(\mathbf{I}-\frac{\mu^2}{2}(t-s)\eta\,\mathbf{I}+\frac{\eta}{2}\int_s^t \nabla_\bTheta^2L(\bZ_r)dr\right)+O(\eta^2).
\end{align*}
\end{proof}

\bibliographystyle{siam}
\bibliography{sample}

\end{document}